\documentclass[10pt,journal,compsoc]{IEEEtran}
\usepackage{amsmath,amsfonts}
\usepackage{amssymb}
\usepackage{amsthm}
\usepackage{mathabx}
\usepackage{bbm}
\usepackage{algorithmic}
\usepackage{algorithm}
\usepackage{array}
\usepackage[caption=false,font=normalsize,labelfont=sf,textfont=sf]{subfig}
\usepackage{textcomp}
\usepackage{stfloats}
\usepackage{url}
\usepackage{verbatim}
\usepackage{graphicx}
\usepackage{cite}
\usepackage{multirow}
\usepackage{epsfig}
\usepackage{booktabs}
\usepackage{dsfont}
\usepackage{multirow}

\usepackage[colorlinks,linkcolor=black,anchorcolor=black,citecolor=black,urlcolor=black]{hyperref}
\newtheorem{theorem}{Theorem}
\newtheorem{proposition}{Proposition}

\usepackage{tikz,xcolor,hyperref}
\definecolor{lime}{HTML}{A6CE39}
\DeclareRobustCommand{\orcidicon}{%
	\begin{tikzpicture}
		\draw[lime, fill=lime] (0,0) 
		circle [radius=0.16] 
		node[white] {{\fontfamily{qag}\selectfont \tiny ID}};    \draw[white, fill=white] (-0.0625,0.095) 
		circle [radius=0.007];    \end{tikzpicture}
	\hspace{-2mm}}
\foreach \x in {A, ..., Z}{%
	\expandafter\xdef\csname orcid\x\endcsname{\noexpand\href{https://orcid.org/\csname orcidauthor\x\endcsname}{\noexpand\orcidicon}}
}

\usepackage[commandnameprefix=always]{changes}
\usepackage{ulem}
\usepackage{cancel}
\begin{document}

\title{Federated Prediction-Powered Inference from Decentralized Data}

\author{Ping Luo, Xiaoge Deng, Ziqing Wen, Tao Sun, Dongsheng Li
\thanks{Manuscript received. \textit{(Corresponding author: Tao Sun, Dongsheng Li )}}
\thanks{Ping Luo, Xiaoge Deng, Ziqing Wen, Tao Sun, Dongsheng Li are with the National Laboratory for Parallel and Distributed Processing, National University of Defense Technology, ChangSha, 410073, China (e-mail: \href{luoping@nudt.edu.cn}{luoping@nudt.edu.cn}; \href{dengxg@nudt.edu.cn}{dengxg@nudt.edu.cn}; \href{zqwen@nudt.edu.cn}{zqwen@nudt.edu.cn}; \href{suntao.saltfish@outlook.com}{suntao.saltfish@outlook.com}; \href{dsli@nudt.edu.cn}{dsli@nudt.edu.cn}).}
}

\markboth{Journal of \LaTeX\ Class Files,~Vol.~14, No.~8, August~2021}%
{Shell \MakeLowercase{\textit{et al.}}: A Sample Article Using IEEEtran.cls for IEEE Journals}

\IEEEpubid{0000--0000/00\$00.00~\copyright~2021 IEEE}

\maketitle

\begin{abstract}
In various domains, the increasing application of machine learning allows researchers to access inexpensive predictive data, which can be utilized as auxiliary data for statistical inference. Although such data are often unreliable compared to gold-standard datasets, Prediction-Powered Inference (PPI) has been proposed to ensure statistical validity despite the unreliability. However, the challenge of `data silos' arises when the private gold-standard datasets are non-shareable for model training, leading to less accurate predictive models and invalid inferences. In this paper, we introduces the Federated Prediction-Powered Inference (Fed-PPI) framework, which addresses this challenge by enabling decentralized experimental data to contribute to statistically valid conclusions without sharing private information. The Fed-PPI framework involves training local models on private data, aggregating them through Federated Learning (FL), and deriving confidence intervals using PPI computation. The proposed framework is evaluated through experiments, demonstrating its effectiveness in producing valid confidence intervals.
\end{abstract}


\section{Introduction}
\IEEEPARstart{A}{s} machine learning is increasingly applied across various domains, researchers can obtain a wealth of inexpensive data from model predictions, such as predictions of protein structures, gene sequences, climate patterns, etc \cite{mirdita2022colabfold,reichstein2019deep,rives2021biological,jaganathan2019predicting}. The utility of model predictions as auxiliary data has been well-established in statistical inference\cite{wu2001model}, and significant efforts have been devoted to developing asymptotically valid confidence intervals when the predictive model is trained on the experimental (gold-standard) datasets to get predictive datasets \cite{breidt2017model}. Although these predictive datasets are often unreliable compared to the gold-standard datasets, Prediction-Powered Inference (PPI) has been proposed to extract information from unreliable predictions while ensuring the statistical validity of the conclusions, which is the smaller confidence intervals $\mathcal{C}^{PP}$ and the powerful $P$ values.\cite{angelopoulos2023prediction}.

The PPI method requires a large gold-standard dataset to train models, enabling more accurate predictions. However, in real-world scenarios, these gold-standard datasets are often considered valuable assets by research institutions and are thus not shared, leading to the problem of `data silos' in the relevant research fields \cite{leonelli2019data,miyakawa2020no,li2022federated}. Under these circumstances, researchers are forced to use private experimental datasets for training, which often have small sample sizes and incomplete features, resulting in less accurate predictive models and invalid statistical inferences \cite{kim2016collaborative,naeem2020reliable,caiafa2020decomposition,koppe2021deep}. Therefore, researchers must devise a strategy that avoids the direct centralization of private data from various research institutions, while still enabling the participation of all data in training and inference \cite{nguyen2022novel}.

\begin{figure}[!t]
	\centering
	\includegraphics[width=0.5\textwidth]{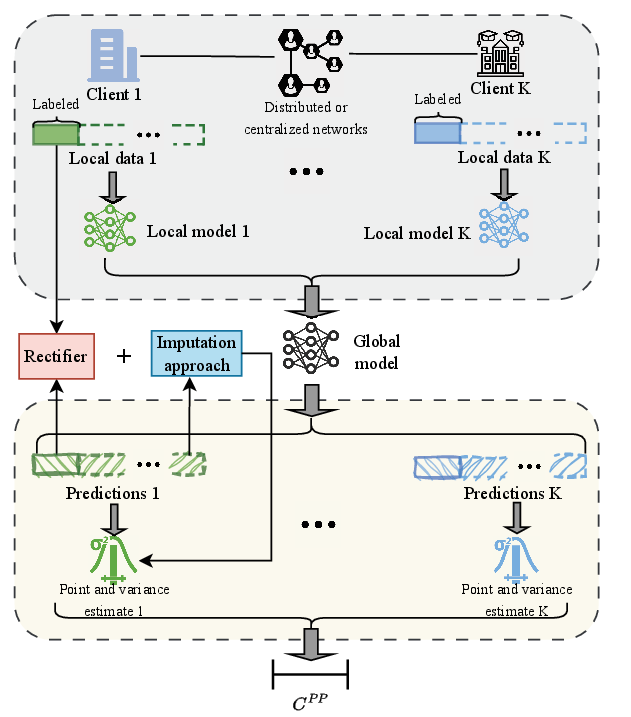}
	\caption{Systems for prediction-powered inference in FL; The upper half of the figure represents the traditional FL training process, while the lower half depicts the Prediction-Powered Inference process and parameters aggregation on the client side.}
	\label{system}
\end{figure}

\IEEEpubidadjcol

In this paper, we first propose the Federated Prediction-Powered Inference (Fed-PPI) framework, which aims to derive statistically valid conclusions from decentralized experimental data without the need to share individual (or institutional) privacy information \cite{mcmahan2017communication,angelopoulos2023prediction}. In Figure \ref{system}, we present the system architecture of Fed-PPI. Initially, each client participating in Federated Learning (FL) training utilizes its local experimental data (typically with unlabeled data significantly outnumbering labeled data) to train local models \cite{mcmahan2017communication}. These local models can be trained using labeled data through supervised learning or all available data through semi-supervised learning \cite{fan2022private,diao2022semifl,pei2022knowledge}. The trained models are then either sent to a central server for aggregation (centralized FL) \cite{mcmahan2017communication} or directly aggregated with neighboring nodes (decentralized FL) \cite{sun2022decentralized} to obtain a global model. Next, each client uses the global model to predict on its local dataset (both labeled and unlabeled) and computes the measure of fit and the rectifier as defined in PPI \cite{angelopoulos2023prediction}. The measure of fit is a statistical value (e.g., mean) or a function used to derive a statistical value (e.g., gradient) based on predictions from the unlabeled data, incorporating the prediction error of the global model. The rectifier, which measures this error, is computed using the labeled data and their corresponding predictions. These measure of fit and rectifier values are then combined on each client to obtain the relevant parameters for the local confidence intervals. Finally, these parameters are sent to the FL aggregation process to derive the confidence interval $\mathcal{C}^{PP}$ for the entire dataset.

The main contributions of this paper are as follows:
\begin{itemize}
	\item	We introduce Federated Learning and develop a new Fed-PPI framework to address the `data silos' problem of PPI in real-world scenarios,. In Fed-PPI, each participating entity (client) trains its model locally, and through the FL aggregation process combined with the PPI computation, obtains the global model and global confidence interval. 
	\item	We define the objectives and processes of the Fed-PPI framework and propose corresponding algorithms for common statistical problems such as means, quantiles, and coefficients in linear and logistic regression. Additionally, we provide a theoretical analysis of these algorithms.
	\item	We conducted experiments using the dataset from \cite{angelopoulos2023prediction} to evaluate the algorithms proposed under the Fed-PPI framework. The results demonstrate that the obtained confidence intervals are statistically valid..
\end{itemize}

This paper is organized as follows. Related work is introduced in Section \ref{sec-related work}. The basics of Fed-PPI is summarized in Section \ref{sec-preliminaries}. The algorithm for common statistical problems are presented in Section \ref{sec-algorithm}. Experimentation results are shown in Section \ref{sec-experiment}. The conclusion is presented in Section \ref{sec-conclusion}.

\section{Related Work}
\label{sec-related work}
FL has emerged as a promising approach for collaborative model training across multiple institutions without the need to share sensitive data, making it particularly suitable for applications in healthcare, biology, chemistry, and materials science. For instance, Dayan et al. (2021) demonstrated the effectiveness of FL in predicting clinical outcomes for COVID-19 patients by integrating data from various healthcare institutions while preserving patient privacy \cite{dayan2021federated}. Similarly, Sheller et al. (2020) explored the feasibility of multi-institutional deep learning for brain tumor segmentation, highlighting the potential of FL to build robust models without centralized data sharing \cite{sheller2020federated}. Xu et al. (2021) provided a comprehensive review of FL in healthcare informatics, discussing its application in developing predictive models using federated electronic health records. \cite{xu2021federated,dang2022federated}. In addition to healthcare, Banabilah et al. (2022) discussed broader FL applications, including its potential impact on fields such as biology, chemistry, and materials science, where data privacy concerns are paramount \cite{banabilah2022federated}. These studies collectively underscore the growing significance of FL in facilitating secure, collaborative research across diverse scientific domains.

As we mentioned in the previous section, applying the FL framework to PPI is a novel attempt. In this endeavor, it is essential to understand the research trajectory of PPI, where it has evolved from a body of work focused on estimation with many unlabeled data points and few labeled data \cite{mey2022improved,azriel2022semi,song2024general,zhang2022high,Wang2020MethodsFC}. Although the original work mentions that PPI can correct biases introduced by model predictions \cite{angelopoulos2023prediction}, the accuracy of these predictions still depends on the degree of model training. However, in FL, models trained in a decentralized manner often struggle to meet the prediction accuracy benchmarks set by centrally trained models, as this depends on the degree of data dispersion and the constraints of communication overhead \cite{cheng2022aafl}. 

FL has seen significant advancements in achieving prediction accuracy comparable to centralized machine learning through the development of various optimization algorithms. The foundational work by McMahan et al. (2017) introduced the Federated Averaging (FedAvg) algorithm, which effectively balances local updates and global aggregation, enabling FL models to perform similarly to centrally trained models \cite{mcmahan2017communication}. Building on this, researchers explored optimization methods like FedProx, FedAMP, FedNova and SCAFFOLD, demonstrating improved accuracy in non-IID data settings, further closing the gap between FL and centralized learning \cite{li2020federated,huang2021personalized,wang2020tackling,karimireddy2020scaffold}. These studies collectively highlight the rapid progress in FL, particularly in optimizing model accuracy across varying data distributions. In Section \ref{sec-experiment}, we will demonstrate that by applying PPI to FL and equipping it with state-of-the-art FL optimization algorithms, our Fed-PPI framework can achieve confidence intervals nearly identical to those obtained through centralized training, all while preserving the privacy of decentralized data.

\section{Preliminaries and Definitions}
\label{sec-preliminaries}
In this section, we introduce fundamental concepts of the PPI method and relevant definitions within the FL-PPI framework. 

\subsection{Convex Estimation}
\label{convex-est}
For the sake of mathematical analysis, we assume that there are $ K $ clients, and the samples on all clients are labeled. That is, there exists a dataset $\left(\widebar{X}_{k}^{i}, \widebar{Y}_{k}^{i},f({\widebar{X}}_k^i)\right) \in (\mathcal{X} \times \mathcal{Y})^{m_k}$, where $i \in [1, m_k]$ and $k \in [1, K]$. The model prediction function $f$ is obtained from the data $(\widebar{X}_k^i, \widebar{Y}_k^i)$ across all clients, trained within the FL framework and maps from the input space $\mathcal{X}$ to the output space $\mathcal{Y}$, i.e., $ f: \mathcal{X} \rightarrow \mathcal{Y} $. Our main objective is a technique for inference on estimands that can be expressed as the solution to a convex optimization problem. Formally, we consider estimands of the form
\begin{equation}
\label{estimand}
    \theta^* = \arg \min_{\theta \in \mathbb{R}^d} \mathbb{E} \left[ \ell_\theta \left( \widebar{X}_k^i, \widebar{Y}_k^i \right) \right]
\end{equation}
where $\theta$ represents the mean or many other quantities of a random outcome over a population of interest, and the loss function $\ell_\theta:\mathcal{X}\times\mathcal{Y}\rightarrow\mathbb{R}$ is convex in $\theta \in \mathbb{R}^d$ for some $d \in \mathbb{N}$. Throughout, we take the existence of $\theta^*$ as given, and as $m_k$ approaches $\infty$, $\theta^*$ gets closer to the true value. If the minimizer is not unique, our method will return a confidence set guaranteed to contain all minimizers. Under mild conditions, convexity ensures that $\theta^*$ can also be expressed as the value solving
\begin{equation}
\label{convex-solution}
    \mathbb{E}\left[g_{\theta^\ast}({\widebar{X}}_k^i,{\widebar{Y}}_k^i)\right]=0
\end{equation}
where $g_{\theta}:\mathcal{X}\times\mathcal{Y}\to\mathbb{R}^{p}$ is a subgradient of 
$\ell_\theta$ with respect to $\theta$. 

Due to the principles of FL, we cannot directly access the total dataset $\bigcup (\widebar{X}_k^i, \widebar{Y}_k^i)$ from individual clients to obtain statistical information. However, the statistical features of the distribution datasets on each client can be combined using statistical methods to obtain confidence intervals within the FL framework. Based on this statistical method we need to define the following rules.
\subsubsection{Aggregation weights}
In the FedAvg algorithm of FL, each client performs a weighted average of its model parameters (or gradient parameters) at each round of aggregation. This weight is related to the sample number of the local dataset on the client and is defined as
    \begin{equation}
    \label{original-weight}
    p_k := \frac{m_k}{\sum_{k=1}^{K} m_k}.
    \end{equation}
We extend this weighting to the aggregation step of combining the statistical parameters of individual clients. The validity of this weighting will be demonstrated in the experiments presented in Section \ref{sec-experiment}.
\subsubsection{Imputed gradient}
\label{impute prediction}
Our objective is to obtain a confidence interval for the estimands $\theta^*$. This requires a substantial amount of labeled data $({\widebar{X}}_k^i, {\widebar{Y}}_k^i)$. However, in practice, we can only obtain a large amount of predicted data $({\widebar{X}}_k^i, f({\widebar{X}}_k^i))$. Therefore, we commence with $({\widebar{X}}_k^i, f({\widebar{X}}_k^i))$ and define
\begin{equation}
\label{original-prediction}
    \begin{aligned}
        g(\theta) &=: \sum_{k=1}^{K} p_k \frac{1}{m_k} \sum_{i=1}^{m_k} g_\theta \left( {\widebar{X}}_k^i, f({\widebar{X}}_k^i) \right) \\
        &= \sum_{k=1}^{K} p_k \mathbb{E}_i \left[g_\theta \left( {\widebar{X}}_k^i, f({\widebar{X}}_k^i) \right) \right] \\
        &= \mathbb{E}_{k} \left[\mathbb{E}_{i} \left[g_\theta \left( {\widebar{X}}_k^i, f({\widebar{X}}_k^i) \right) \right]\right]
    \end{aligned}
\end{equation}
For Eq. \eqref{original-prediction}, the $\mathbb{E}_i$ term represents the FL local computation on the clients, while the $\mathbb{E}_k$ term represents the FL global aggregation operation. In traditional approaches, the datasets from individual clients are aggregated for centralized computation as
\begin{equation}
\label{additional-agg}
    g(\theta) =: \mathbb{E} \left[ \bigcup g_\theta (\widebar{X}_k^i, \widebar{Y}_k^i) \right]
\end{equation}
For the sake of brevity, we refer to the two aggregations, Eq. \eqref{original-prediction} and Eq. \eqref{additional-agg}, as $\mathbb{E}_{k,i}$ and $\mathbb{E}_{\bigcup}$, respectively. For imputed predictions, we have $\mathbb{E}_{k,i} = \mathbb{E}_{\bigcup}$. It can be demonstrated that the aggregated parameters accurately represent the centralized data $\bigcup (\widebar{X}_k^i, \widebar{Y}_k^i)$. The proof of this statement is provided in Appendix \ref{appendix1}.

In particular, for every $\theta$, we want a confidence set $\mathcal{T}_\delta(\alpha-\delta)$, satisfying
\begin{equation}
    P\left(g(\theta)\in\mathcal{T}_{\alpha-\delta}(\theta)\right)\geq1-(\alpha-\delta)
\end{equation}
\subsubsection{Empirical rectifier}
The rectifier captures a notion of prediction error. In the general setting of convex estimation problems, the relevant notion of error is the bias of the subgradient $g_\theta$ computed using the predictions:
\begin{equation}
\label{original-rectifier}
    \begin{aligned}
        \Delta(\theta) &=: \sum_{k=1}^{K} p_k \mathbb{E}_i \left[ g_\theta({\widebar{X}}_k^i, {\widebar{Y}}_k^i) - g_\theta({\widebar{X}}_k^i, f({\widebar{X}}_k^i)) \right] \\
        &= \mathbb{E}_{k,i} \left[ g_\theta({\widebar{X}}_k^i, {\widebar{Y}}_k^i) - g_\theta({\widebar{X}}_k^i, f({\widebar{X}}_k^i)) \right]
    \end{aligned}
\end{equation}
For the analysis of $\mathbb{E}_{k,i}$ and $\mathbb{E}_{\bigcup}$, Eq. \eqref{original-rectifier} leads to the same conclusion as Eq. \eqref{original-prediction}.

The next step is to create a confidence set for the rectifier, $\mathcal{R}_\delta(\theta)$, satisfying
\begin{equation}
    P(\Delta(\theta)\in\mathcal{R}_\delta(\theta))\geq1-\delta
\end{equation}
Because $\Delta$ and $g(\theta)$ is an expectation for each $\theta$, $\mathcal{T}_{\alpha-\delta}$ and $\mathcal{R}_\delta(\theta)$ can be constructed using standard, off-the-shelf
confidence intervals for the mean, which we review in Appendix E.

We reformulate the objective of Eq. \eqref{estimand} and Eq. \eqref{convex-solution} to finding the value of $\theta^*$ that satisfies $g(\theta) + \Delta(\theta) = 0$ based on the above definitions. Consequently, we present the following theorem.
\begin{theorem}[Convex estimation]
\label{convex estimation}
    Suppose that the convex estimation problem is nondegenerate as in \eqref{convex-solution}. Fix $\alpha \in (0,1)$ and $\Delta(\theta) \in (0,\alpha)$. Suppose that, for any $\theta \in \mathbb{R}^d$, we can construct $\mathcal{T}_{\alpha-\delta}$ and $\mathcal{R}_\delta(\theta)$ satisfying
    \begin{equation}
        \left\{
            \begin{aligned}
                &P\left(g(\theta)\in\mathcal{T}_{\alpha-\delta}(\theta)\right)\geq1-(\alpha-\delta)\\
                &P(\Delta(\theta)\in\mathcal{R}_\delta(\theta))\geq1-\delta
            \end{aligned} 
        \right.
    \end{equation}
    Let $\mathcal{C}_\alpha^{PP}=\{\theta:0\in\mathcal{R}_\delta(\theta)+\mathcal{T}_{\alpha-\delta}(\theta)\}$, where $+$ denotes the Minkowski sum. Then,
    \begin{equation}
        P(\theta^\ast\in\mathcal{C}_\alpha^{PP})\geq1-\alpha
    \end{equation}
\end{theorem}
This result means that we can construct a valid confidence set for $\theta^*$, without assumptions about the data distribution or the machine-learning model, for any nondegenerate convex estimation problem. 

\subsection{Actual Estimate}
\label{actual estimate}
In practical scenarios, the dataset on each client typically consists of both labeled and unlabeled samples. The labeled data for each client is denoted as $(X_k, Y_k) \in (\mathcal{X} \times \mathcal{Y})^{n_k}$, where $X_k = (X_k^1, \ldots, X_k^{n_k})$ and $Y_k = (Y_k^1, \ldots, Y_k^{n_k})$. Additionally, each client possesses unlabeled data denoted as $(\widetilde{X}_k, \widetilde{Y}_k) \in (\mathcal{X} \times \mathcal{Y})^{N_k}$, where $\widetilde{X}_k = (\widetilde{X}_k^1, \ldots, \widetilde{X}_k^{N_k})$ and $\widetilde{Y}_k = (\widetilde{Y}_k^1, \ldots, \widetilde{Y}_k^{N_k})$. It is assumed that $ N_k \gg n_k $ for all clients, and that the labels of the data on each client, including the predictions, conform to a normal distribution. Notably, $\widetilde{Y}_k$ represents the predicted output of $\widetilde{X}_k$ after being processed by the FL-trained model, and thus cannot be derived from direct observation. 

Compared to the dataset $(\widebar{X}_{k}^{i}, \widebar{Y}_{k}^{i}, f({\widebar{X}}_k^i))$, we will redefine the `Aggregation weights' and use the current dataset to estimate the `Imputed prediction' and `Empirical rectifier'.
\subsubsection{Aggregation Weights}
Since we have divided the dataset $(\widebar{X}_{k}^{i}, \widebar{Y}_{k}^{i}, f({\widebar{X}}_k^i))$ into two parts—labeled and unlabeled—with sample sizes $n_k$ and $N_k$, respectively, where $m_k = n_k + N_k$, the aggregation weights $p_k$ in \eqref{original-weight} can be redefined as follows:
\begin{equation}
    \label{practical-weight}
    p_k := \frac{n_k + N_k}{\sum_{k=1}^{K} (n_k + N_k)}.
\end{equation}
\subsubsection{Imputed gradient}
For Eq. \eqref{original-prediction}, we estimate it directly with the unlabeled dataset $(\widetilde{X}_k, \widetilde{Y}_k)$.
\begin{equation}
\label{practical-prediction}
    \widetilde{g}(\theta) =: \sum_{k=1}^{K} p_k \frac{1}{N_k} \sum_{i=1}^{N_k} g_\theta\big(\widetilde{X}_k^i, f(\widetilde{X}_k^i)\big)
\end{equation}
For Eq. \eqref{unconvex-prediction}, we use the first half of the unlabeled dataset $(\widetilde{X}_k, \widetilde{Y}_k)$ (assuming $N_k$ is even) to estimate the value of $\theta^\ast$, and define

\subsubsection{Empirical rectifier}
We use the labeled dataset $(X_k, Y_k)$ to estimate the rectifier. Consequently, Eq. \eqref{original-rectifier} and Eq. \eqref{unconvex-rectifier} can be replaced by:
\begin{equation}
\label{practical-rectifier}
    \widehat{\Delta}(\theta) =: \sum_{k=1}^{K} p_k \frac{1}{n_k} \sum_{i=1}^{n_k} \left( g_\theta(X_k^i, Y_k^i) - g_\theta(X_k^i, f(X_k^i)) \right)
\end{equation}
The following is an asymptotic counterpart of Theorem \ref{convex estimation} that uses the central limit theorem in the confidence set construction.
\begin{theorem}[Convex estimation: asymptotic version]
\label{asymptotic}
    Suppose that the convex estimation problem is nondegenerate as in \eqref{convex-solution}. Denoting by $g^j(x, y)$ the $j$-th coordinate of $g(x, y)$. Fix $\alpha \in (0,1)$ and $j \in [d]$. For all $\theta \in \mathbb{R}^d$, define
    \begin{equation}
        \left\{
            \begin{aligned}
                &\widetilde{g}^{j}(\theta) =: \sum_{k=1}^{K} p_k \frac{1}{N_k} \sum_{i=1}^{N_k} g_\theta\big(\widetilde{X}_k^{i,j}, f(\widetilde{X}_k^{i,j})\big)\\
                &\widehat{\Delta}^{j}(\theta) =: \sum_{k=1}^{K} p_k \frac{1}{n_k} \sum_{i=1}^{n_k} \left( g_\theta(X_k^{i,j}, Y_k^{i,j}) - g_\theta(X_k^{i,j}, f(X_k^{i,j})) \right)
            \end{aligned} 
        \right.
    \end{equation}
Further, define $\left({\widehat{\sigma}}_{g}^j\left(\theta\right)\right)^2$ be the variance of $g_\theta\big(\widetilde{X}_k^i, f(\widetilde{X}_k^i)\big)$ values, and $\left({\widehat{\sigma}}_\Delta^j(\theta)\right)^2$ be the variance of $ g_\theta(X_k^i, Y_k^i) - g_\theta(X_k^i, f(X_k^i))$ values. Let $w_\alpha^j(\theta) = z_{1-\alpha/(2p)} \sqrt{\frac{\left({\widehat{\sigma}}_g^j\left(\theta\right)\right)^2}{N} + \frac{\left({\widehat{\sigma}}_\Delta^j(\theta)\right)^2}{n}}$ and $ \mathcal{C}_\alpha^{PP} = \left\{ \theta : \left| {\widetilde{g}}^j(\theta) + {\widehat{\Delta}}^j(\theta) \right| \le w_\alpha^j(\theta), \forall j \in [d] \right\} $.

Then, we have
$$\liminf_{n,N\to\infty}P(\theta^*\in\mathcal{C}_\alpha^{\mathrm{PP}})\geq1-\alpha.$$
\end{theorem}

\subsection{Beyond Convex Estimation}
\label{non-convex-est}
The tools developed in Section \ref{convex-est} were tailored to unconstrained convex optimization problems. In general, however, inferential targets can be defined in terms of nonconvex losses or they may have (possibly even nonconvex) constraints. For such general optimization problems, we cannot expect the condition \eqref{estimand} to hold. We generalize our approach to a broad class of risk minimizers:
\begin{equation}
    \theta^\ast = \arg\min_{\theta \in \Theta} \mathbb{E} \left[ \ell_\theta(\widebar{X}_k^i, \widebar{Y}_k^i) \right]
\end{equation}
where $\ell_\theta:\mathcal{X}\times\mathcal{Y}\rightarrow\mathbb{R}$ is a possibly nonconvex loss function and $\Theta$ is an arbitrary set of admissible parameters.
As before, if $\theta^*$ is not a unique minimizer, our method will return a set that contains all minimizers.

In the following, we continue to use the finite dataset for our analysis, $(X_k, Y_k) \in (\mathcal{X} \times \mathcal{Y})^{n_k}$ and $(\widetilde{X}_k, \widetilde{Y}_k) \in (\mathcal{X} \times \mathcal{Y})^{N_k}$, ensuring that the aggregation weights on the individual clients remain as defined in Equation \eqref{practical-weight}.
\subsubsection{Imputed gradient}
Since we don't know the value of $\mathbb{E}_{k,i}\left[\ell_{\theta^\ast}(X_k^i,Y_k^i)\right]$, we use the first half of the unlabeled data (assuming $N_k$ is even) to estimate the value of $\theta^\ast$, and define
\begin{equation}
\label{unconvex-prediction}
        \left\{
            \begin{aligned}
                &\widetilde{\theta}^f = \arg\min_{\theta \in \Theta} \sum_{k=1}^{K} p_k \frac{2}{N_k} \sum_{i=1}^{N_k/2} \ell_\theta\big(\widetilde{X}_k^i, f(\widetilde{X}_k^i)\big)\\
                &\widetilde{L}^f(\theta) := \sum_{k=1}^{K} p_k \frac{2}{N_k} \sum_{i=N_k/2+1}^{N_k} \ell_\theta\big(\widetilde{X}_k^i, f(\widetilde{X}_k^i)\big)
            \end{aligned} 
        \right.
    \end{equation}
\subsubsection{Empirical rectifier}
To correct the imputation approach, we rely on the following rectifier:
\begin{equation}
\label{unconvex-rectifier}
    \widehat{\Delta}(\theta) =: \sum_{k=1}^{K} p_k \frac{1}{n_k} \sum_{i=1}^{n_k} \left( \ell_\theta(X_k^i, Y_k^i) - \ell_\theta(X_k^i, f(X_k^i)) \right)
\end{equation}
\begin{theorem}[General risk minimization: finite population]
\label{risk minim}
    Fix $\alpha \in (0,1)$ and $\Delta(\theta) \in (0,\alpha)$. Suppose that, for any $\theta \in \Theta$, we can construct $\left(\mathcal{R}_{\delta/2}^l(\theta),\mathcal{R}_{\delta/2}^u(\theta)\right)$ and $ \left(\mathcal{T}_{\frac{\alpha-\delta}{2}}^l(\theta),\mathcal{T}_{\frac{\alpha-\delta}{2}}^u(\theta)\right)$ such that
    \begin{equation}
        \left\{
            \begin{aligned}
                &P\big(\Delta(\theta) \le \mathcal{R}_{\delta/2}^u(\theta)\big) \geq 1 - \delta/2\\
                &P\big(\Delta(\theta) \geq \mathcal{R}_{\delta/2}^l(\theta)\big) \geq 1 - \delta/2
            \end{aligned} 
        \right.
    \end{equation}
    and 
    $$
        \left\{
            \begin{aligned}
                &P\big(\widetilde{L}^f(\theta) - \mathbb{E}_{k,i} \left[\ell_\theta(\widetilde{X}_k^i, f({\widetilde{X}}_k^i))\right] \le \mathcal{T}_{\frac{\alpha - \delta}{2}}^u(\theta)\big) \geq 1 - \frac{\alpha - \delta}{2}\\
                &P\big(\widetilde{L}^f(\theta) - \mathbb{E}_{k,i} \left[\ell_\theta(\widetilde{X}_k^i, f({\widetilde{X}}_k^i))\right] \geq \mathcal{T}_{\frac{\alpha - \delta}{2}}^l(\theta)\big) \geq 1 - \frac{\alpha - \delta}{2}
            \end{aligned} 
        \right.
    $$
    Let 
    $$\mathcal{R}_{\delta/2}^d(\theta)=\mathcal{R}_{\delta/2}^u(\widetilde{\theta}^f)- \mathcal{R}_{\delta/2}^l(\theta),$$
    $$\mathcal{T}_{\frac{\alpha - \delta}{2}}^d(\theta)=\mathcal{T}_{\frac{\alpha - \delta}{2}}^u(\theta) - \mathcal{T}_{\frac{\alpha - \delta}{2}}^l(\widetilde{\theta}^f)$$
    $$
        \mathcal{C}_\alpha^{PP} = \left\{ \theta \in \Theta : \widetilde{L}^f(\theta) \leq L^f(\widetilde{\theta}^f) + \mathcal{R}_{\delta/2}^d(\theta) + \mathcal{T}_{\frac{\alpha - \delta}{2}}^d(\theta) \right\}
    $$
    Then, we have
    $$P(\theta^* \in \mathcal{C}_\alpha^{PP}) \geq 1 - \alpha$$
\end{theorem}

\section{Algorithms}
\label{sec-algorithm}
In this section, we present FL-PPI algorithms for several canonical inference problems. These corresponding algorithms will be designed on dataset $(X_k, Y_k)$ and $(\widetilde{X}_k, \widetilde{Y}_k)$ (in Section \ref{actual estimate}) at each client $k \in [1, K]$. 

\subsection{Example: Mean Estimation}
\label{example}
Before presenting our main results, we use the example of mean estimation to build intuition. Our goal is to give a valid confidence interval for the average outcome, $\theta^\ast = \mathbb{E}_{k,i}[\widebar{Y}_k^i], i \in [1, m_k]$. We construct a prediction-powered estimate for each client, $\widehat{\theta}_k^{PP}$, and show that it leads to tighter confidence intervals $\mathcal{C}_\alpha^{PP}$. Consider
\begin{equation}
    \widehat{\theta}_k^{PP} = \underbrace{\frac{1}{N_k} \sum_{i=1}^{N_k} f(\widetilde{X}_k^i)}_{\widetilde{\theta}_k^f} - \underbrace{\frac{1}{n_k} \sum_{i=1}^{n_k} (f(X_k^i) - Y_k^i)}_{\widehat{\Delta}_k}
\end{equation}
After calculating the mean estimate on each client, we need to aggregate the following parameters
\subsubsection{Estimands aggregation}
\begin{equation}
    \widehat{\theta}^{PP} = \sum_{k=1}^{K} p_k \widehat{\theta}_k^{PP}
\end{equation}
\subsubsection{Predictions and rectifiers aggregation}
\begin{equation}
\label{FL-pre-rec}
    \widetilde{\theta}^f = \sum_{k=1}^{K} p_k \widetilde{\theta}_k^f; \quad \widehat{\Delta}(\theta) = \sum_{k=1}^{K} p_k \widehat{\Delta}_k(\theta)
\end{equation}
thus we hava $\widehat{\theta}^{PP} = \widetilde{\theta}^f - \widehat{\Delta}(\theta)$.
\subsubsection{Variances aggregation}
\begin{equation}
\label{var-agg}
    \left\{
            \begin{aligned}
                &\left(\widehat{\sigma}^f\right)^2 = \sum_{k=1}^{K} p_k \left( (\widehat{\sigma}_k^f)^2 + (\widetilde{\theta}_k^f - \widetilde{\theta}^f)^2 \right)\\
                &\left(\widehat{\sigma}^{f-Y}\right)^2 = \sum_{k=1}^{K} p_k \left( (\widehat{\sigma}_k^{f-Y})^2 + (\widehat{\Delta}_k(\theta) - \widehat{\Delta}(\theta))^2 \right)
            \end{aligned} 
        \right.
\end{equation}
where $(\widehat{\sigma}_k^f)^2$ and $(\widehat{\sigma}_k^{f-Y})^2$ are the estimated variances of the $f(\widetilde{X}_k^i)$ and $f(X_k^i) - Y_k^i$ at client $k$, respectively. When the datasets on each client are IID, we have $\widetilde{\theta}_k^f \approx \widetilde{\theta}^f$ and $\widehat{\Delta}_k(\theta) \approx \widehat{\Delta}(\theta)$. In this case, Eq. \eqref{var-agg} can be interpreted as a weighted average of the variances across the clients, which does not introduce any additional bias \cite{mood1950introduction}.

Notice $\widehat{\theta}^{PP}$ is unbiased for $\theta^\ast$ and it is a sum of two independent terms $\widetilde{\theta}^f$ and $\widehat{\Delta}$. Thus, we can construct 95\% confidence intervals for $\theta^\ast$ as
\begin{equation}
\label{eq-mean-example}
    \mathcal{C}_\alpha^{PP} = \underbrace{\widehat{\theta}^{PP} \pm 1.96 \sqrt{ \frac{\left(\widehat{\sigma}^f\right)^2}{N} + \frac{\left(\widehat{\sigma}^{f-Y}\right)^2}{n} }}_{FL\ prediction-powered\ interval}
\end{equation}
where $N=\sum_{k=1}^{K} N_k $ and $n = \sum_{k=1}^{K} n_k$. According to \cite{angelopoulos2023prediction}, when $N \gg n$, the width of the $\mathcal{C}_\alpha^{PP}$ depends on $\left(\widehat{\sigma}^{f-Y}\right)^2$. Therefore, in general, since $n_k < n$, the $\mathcal{C}_\alpha^{PP}$ on each client tends to be wider than the $\mathcal{C}_\alpha^{PP}$ on the total dataset. Additionally, if $\left(\widehat{\sigma}^{f-Y}\right)^2$ can accurately represent the rectifier on the total dataset (i.e., $\mathbb{E}_{k,i} = \mathbb{E}_{\bigcup}$), the FL aggregation will approach the $\mathcal{C}_\alpha^{PP}$ width and $\widehat{\theta}^{PP}$ of the total dataset.

\subsection{Proposition for Algorithms}
We can express the process of solving the estimand as solving a convex function problem using algorithms such as mean, quantile, logistic regression, and linear regression. The corresponding algorithms and propositions are as follows:
\subsubsection{Mean estimation}
We begin by returning to the problem of mean estimation:
\begin{equation}
\label{mean-goal}
    \theta^\ast = \mathbb{E}_{k,i}[\widebar{Y}_k^i], 
\end{equation}
where $i \in [1, m_k]$. This objective can be transformed into a convex optimization problem for the mean according to Eq. \eqref{estimand}, i.e., the function $\ell_\theta$ can be expressed as the minimizer of the average squared loss:

$$\theta^* = \arg \min_{\theta \in \mathbb{R}} \mathbb{E}_{k,i} \left[ \ell_\theta \left( \widebar{Y}_k^i \right) \right]=\arg \min_{\theta \in \mathbb{R}} \mathbb{E}_{k,i}\left[ \frac{1}{2}(\theta - \widebar{Y}_k^i)^2\right]$$

The squared loss $\ell_\theta \left( \widebar{Y}_k^i \right)$ is differentiable, with gradient equal to $g_\theta \left( \widebar{Y}_k^i \right) = \theta - \widebar{Y}_k^i$. Applying this in the definition of the prediction \eqref{practical-prediction} and rectifier \eqref{practical-rectifier}, we obtain $\widetilde{g}(\theta) = \theta-\mathbb{E}_{k,i}\left[f(\widetilde{X}_k^i)\right]$ and $\widehat{\Delta}(\theta) = \mathbb{E}_{k,i}\left[f(X_k^i) - Y_k^i\right]$. Based on this, we provide an explicit algorithm for prediction-powered mean estimation and its guarantee in Algorithm \ref{FL-mean-alg} and Proposition \ref{mean-prop}, respectively.
\begin{figure}[!t]
	\begin{algorithm}[H]
		\caption{FL-prediction-powered mean estimation}
		\label{FL-mean-alg}
		\textbf{Input}: Labeled data $(X_k^i, Y_k^i)$, unlabeled features $\widetilde{X}_k^i$, datasize $\{N_k,N,n_k,n\}$, predictor $f$, error level $\alpha \in (0,1)$.
		
		\begin{algorithmic}[1] 
			\STATE Prediction-powered estimator:\\ ${\widetilde{\theta}}^f \gets \sum_{k=1}^{K} p_k {\widetilde{\theta}}_k^f = \sum_{k=1}^{K} {p_k \frac{1}{N_k} \sum_{i=1}^{N_k} f({\widetilde{X}}_k^i)}$.
			\STATE Empirical rectifier: $\widehat{\Delta}(\theta) \gets \sum_{k=1}^{K} p_k {\widehat{\Delta}}_k(\theta)$ \\ $= \sum_{k=1}^{K} {p_k \frac{1}{n_k} \sum_{i=1}^{n_k} (f(X_k^i) - Y_k^i)}$.
			\STATE Prediction-powered estimator: ${\widehat{\theta}}^{PP} \gets {\widetilde{\theta}}^f - \widehat{\Delta}(\theta)$.
                \STATE Empirical variance of prediction at client $k$:\\$\left({\widehat{\sigma}}_k^f\right)^2 \gets \frac{1}{N_k} \sum_{i=1}^{N_k} (f({\widetilde{X}}_k^i) - {\widetilde{\theta}}_k^f)^2$
                \STATE Empirical variance of rectifier at client $k$:\\$\left({\widehat{\sigma}}_k^{f-Y}\right)^2 \gets \frac{1}{n_k} \sum_{i=1}^{n_k} (f(X_k^i) - Y_k^i - {\widehat{\Delta}}_k(\theta))^2$
			\STATE Aggregate variances from all client:\\$\left({\widehat{\sigma}}^f\right)^2 \gets \sum_{k=1}^{K} {p_k (\left({\widehat{\sigma}}_k^f\right)^2 + ({\widetilde{\theta}}_k^f - {\widetilde{\theta}}^f)^2)}$.\\
            $\left({\widehat{\sigma}}^{f-Y}\right)^2 \gets \sum_{k=1}^{K} {p_k (\left({\widehat{\sigma}}_k^{f-Y}\right)^2 + ({\widehat{\Delta}}_k(\theta) - \widehat{\Delta}(\theta))^2)}$
                \STATE $w_\alpha \gets z_{1-\alpha/2} \sqrt{\frac{\left({\widehat{\sigma}}^f\right)^2}{N} + \frac{\left({\widehat{\sigma}}^{f-Y}\right)^2}{n}}$
		\end{algorithmic}
            \textbf{Output}: FL-prediction-powered confidence set $ \mathcal{C}_\alpha^{PP} = \left( \widehat{\theta}^{PP} \pm w_\alpha \right) $
	\end{algorithm}
\end{figure}

\begin{proposition}[Mean estimation]
\label{mean-prop}
    Let $\theta^*$ be the mean outcome \eqref{mean-goal}. Then, the prediction-powered confidence interval in Algorithm \ref{FL-mean-alg} has valid coverage:
$$
    \liminf_{n,N\to\infty} P\left(\theta^* \in \mathcal{C}_\alpha^{PP}\right) \geq 1 - \alpha.
$$
\end{proposition}
\subsubsection{Quantile estimation}
We now turn to quantile estimation. For a pre-specified level $q \in (0, 1)$, we wish
to estimate the q-quantile of the outcome distribution:
\begin{equation}
\label{quantile-goal}
    \theta^* = \min \{\theta : P(\widebar{Y}_k^i \leq \theta) \geq q\}.
\end{equation}
It is well known [36] that the q-quantile can be expressed in variational form as
\begin{equation}
\begin{aligned}
    \theta^* &= \arg\min_{\theta \in \mathbb{R}} \mathbb{E}_{k,i} [\ell_{\theta}(\widebar{Y}_k^i)]\\
    &= \arg\min_{\theta \in \mathbb{R}} \mathbb{E}_{k,i} [q(\widebar{Y}_k^i - \theta) \mathbbm{1}\{\widebar{Y}_k^i > \theta\} \\
    &\quad + (1 - q)(\theta -\widebar{Y}_k^i) \mathbbm{1}\{\widebar{Y}_k^i \leq \theta\}]
\end{aligned}
\end{equation}
where $\ell_{\theta}$ is called the quantile loss. The quantile loss has subgradient $ g_\theta(\widebar{Y}_k^i) = -q \mathbbm{1} \{\widebar{Y}_k^i > \theta\} + (1 - q)\mathbbm{1}\{\widebar{Y}_k^i \leq \theta\} = -q + \mathbbm{1}\{\widebar{Y}_k^i \leq \theta\} $. Applying this in the definition of the prediction \eqref{practical-prediction} and rectifier \eqref{practical-rectifier}, we obtain $\widetilde{g}(\theta) =\mathbb{E}_{k,i}\left[\mathbbm{1}\left\{f({\widetilde{X}}_k^i) \le \theta \right\}\right]-q$ and $\widehat{\Delta}(\theta) = \mathbb{E}_{k,i}\left[\mathbbm{1}\left\{Y_k^i \le \theta \right\} - \mathbbm{1}\left\{f(X_k^i) \le \theta \right\} \right]$. In Algorithm \ref{FL-quantile-alg} we state
an algorithm for FL-prediction-powered quantile estimation; see Proposition \ref{quantile-prop} for a statement of validity.
\begin{figure}[!t]
	\begin{algorithm}[H]
		\caption{FL-prediction-powered quantile estimation}
		\label{FL-quantile-alg}
		\textbf{Input}: Labeled data $(X_k^i, Y_k^i)$, unlabeled features $\widetilde{X}_k^i$, datasize $\{N_k,N,n_k,n\}$, predictor $f$, quantile $q \in (0,1)$, error level $\alpha \in (0,1)$.
  
		\begin{algorithmic}[1] 
                \STATE Construct fine grid $\Theta_{grid}$ between $\min_{k,i}f({\widetilde{X}}_k^i)$ and $\max_{k,i}f({\widetilde{X}}_k^i)$.
                \FOR{$\theta \in \Theta_{grid}$}
			\STATE Imputed CDF: $\widetilde{F}(\theta) \gets \sum_{k=1}^{K} p_k \widetilde{F}_k(\theta) = \sum_{k=1}^{K} p_k$\\
                $ \frac{1}{N_k} \sum_{i=1}^{N_k} \mathbbm{1}\left\{ f({\widetilde{X}}_k^i) \le \theta \right\}$.
			\STATE Empirical rectifier: $\widehat{\Delta}(\theta) \gets \sum_{k=1}^{K} p_k \widehat{\Delta}_k(\theta) = $\\ $\sum_{k=1}^{K} p_k\frac{1}{n_k} \sum_{i=1}^{n_k} \left( \mathbbm{1}\left\{ Y_k^i \le \theta \right\} - \mathbbm{1}\left\{ f(X_k^i) \le \theta \right\} \right)$.
                \STATE Empirical variance of CDF at client $k$:$\left({\widehat{\sigma}}_{g_k}\left(\theta\right)\right)^2 \gets $\\
                $\frac{1}{N_k} \sum_{i=1}^{N_k} \left( \mathbbm{1}\left\{ f({\widetilde{X}}_k^i) \le \theta \right\} - \widetilde{F}_k(\theta) \right)^2$
                \STATE Empirical variance of rectifier at client $k$:$\left(\widehat{\sigma}_{\Delta_k}(\theta)\right)^2 \gets $\\
                $\frac{1}{n_k} \sum_{i=1}^{n_k} \left( \mathbbm{1}\left\{ Y_k^i \le \theta \right\} - \mathbbm{1}\left\{ f(X_k^i) \le \theta \right\} - \widehat{\Delta}_k(\theta) \right)^2$
			\STATE Aggregate variances from all client:\\$\left( {\widehat{\sigma}}_g \left( \theta \right) \right)^2 \gets \sum_{k=1}^{K} p_k \Big( \big( {\widehat{\sigma}}_{g_k} \left( \theta \right) \big)^2 + \big( {\widetilde{F}}_k(\theta) - \widetilde{F}(\theta) \big)^2 \Big)$.\\
            $\left( \widehat{\sigma}_{\Delta}(\theta) \right)^2 \gets \sum_{k=1}^{K} p_k \Big( \big( \widehat{\sigma}_{\Delta_k}(\theta) \big)^2 + \big( \widehat{\Delta}_k(\theta) - \widehat{\Delta}(\theta) \big)^2 \Big)$
                \STATE $w_\alpha(\theta) \gets z_{1-\alpha/2} \sqrt{ \frac{ \left( \widehat{\sigma}_g(\theta) \right)^2 }{N} + \frac{ \left( \widehat{\sigma}_{\Delta}(\theta) \right)^2 }{n} }$
                \ENDFOR
		\end{algorithmic}
            \textbf{Output}: FL-prediction-powered confidence set $ \mathcal{C}_\alpha^{PP}=\left\{\theta\in\Theta_{grid} : \left|\widetilde{F}(\theta)+\widehat{\Delta}(\theta)-q\right|\le w_\alpha\left(\theta\right)\right\} $
	\end{algorithm}
\end{figure}
\begin{proposition}[Quantile estimation]
\label{quantile-prop}
    Let $\theta^*$ be the q-quantile \eqref{quantile-goal}. Then, the prediction-powered confidence interval in Algorithm \ref{FL-quantile-alg} has valid coverage:
$$
    \liminf_{n,N\to\infty} P\left(\theta^* \in \mathcal{C}_\alpha^{PP}\right) \geq 1 - \alpha.
$$
\end{proposition}
\subsubsection{Logistic regression}
In logistic regression, the target of inference is defined by
\begin{equation}
\label{logistic-goal}
    \begin{aligned}
        \theta^* &= \arg\min_{\theta \in \mathbb{R}^d} \mathbb{E}_{k,i}[\ell_{\theta}(\widebar{X}_k^i, \widebar{Y}_k^i)]\\ 
        &= \arg\min_{\theta \in \mathbb{R}^d} \mathbb{E}_{k,i}\left[-\widebar{Y}_k^i \theta^T \widebar{X}_k^i + \log(1 + \exp(\theta^T \widebar{X}_k^i))\right]
    \end{aligned} 
\end{equation}
where $\widebar{Y}_k^i \in {0,1}$. The logistic loss is differentiable and hence the optimality condition \eqref{convex-solution} is ensured. Its gradient is equal to $ g_{\theta}(x, y) = -yx + x\mu_{\theta}(x) $, where $ \mu_{\theta}(x) = \frac{1}{1 + \exp(-x^{\top}\theta)} $ is the predicted mean for point $ x \in \widebar{X} $ based on parameter vector $ \theta $. Applying this in the definition of the prediction \eqref{practical-prediction} and rectifier \eqref{practical-rectifier}, we obtain $\widetilde{g}(\theta)=\mathbb{E}_{k,i}\left[\widetilde{X}_k^{(i,j)}\big(\mu_\theta(\widetilde{X}_k^i)-f({\widetilde{X}}_k^i)\big)\right]$ and $\widehat{\Delta}=\mathbb{E}_{k,i}\left[X_k^{(i,j)}\left(f(X_k^i) - Y_k^i\right)\right]$, where we use $X_k^{(i,j)}$ to denote the $j$-th coordinate of point $X_k^i$. In Algorithm \ref{FL-logistic-alg} we state a method for FL-prediction-powered logistic regression and in Proposition \ref{logistic-prop} we provide its guarantee.
\begin{figure}[!t]
	\begin{algorithm}[H]
		\caption{FL-prediction-powered logistic regression estimation}
		\label{FL-logistic-alg}
		\textbf{Input}: Labeled data $(X_k^i, Y_k^i)$, unlabeled features $\widetilde{X}_k^i$, datasize $\{N_k,N,n_k,n\}$, predictor $f$, error level $\alpha \in (0,1)$.
		\begin{algorithmic}[1] 
                \STATE Construct fine grid $\Theta_{grid}\subset\mathbb{R}^d$ of possible coefficients.
                \STATE Empirical rectifier:${\widehat{\Delta}}^j(\theta) \gets \sum_{k=1}^{K} p_k {\widehat{\Delta}}_k^j (\theta) = \sum_{k=1}^{K} p_k$ \\ 
                $\frac{1}{n_k} \sum_{i=1}^{n_k} X_k^{(i,j)}\left(f(X_k^i) - Y_k^i\right), j\in[d]$
                \STATE Empirical variance of rectifier at client $k$:\\
                $\big(\sigma_{\Delta_k^j }(\theta)\big)^2 \gets \frac{1}{n_k} \sum_{i=1}^{n_k} \left(X_k^{(i,j)} \left(f(X_k^i) - Y_k^i\right) - \Delta_k^j(\theta) \right)^2$
                \FOR{$\theta \in \Theta_{grid}$}
			\STATE Imputed gradient: $\widetilde{g}^j(\theta) \gets \sum_{k=1}^{K} p_k \widetilde{g}_k^j(\theta) = \sum_{k=1}^{K} p_k \frac{1}{N_k} \sum_{i=1}^{N_k} \widetilde{X}_k^{(i,j)} \left(\mu_\theta(\widetilde{X}_k^i) - f(\widetilde{X}_k^i)\right), \mu_\theta(x) = \frac{1}{1 + \exp(-x^\top \theta)}$.
                \STATE Empirical variance of prediction at client $k$:\\
                $\left({\widehat{\sigma}}_{g_k}^j\left(\theta\right)\right)^2 \gets \frac{1}{N_k} \sum_{i=1}^{N_k} \left(\widetilde{X}_k^{(i,j)} (\mu_\theta(\widetilde{X}_k^i) - f(\widetilde{X}_k^i)) - \widetilde{g}_k^j(\theta)\right)^2$
			\STATE Aggregate variances from all client:\\$\left({\widehat{\sigma}}_g^j\left(\theta\right)\right)^2 \gets \sum_{k=1}^{K} p_k \Big(\left({\widehat{\sigma}}_{g_k}^j\left(\theta\right)\right)^2 + \big(\widetilde{g}_k^j(\theta) - \widetilde{g}^j(\theta)\big)^2\Big)$.\\
            $\left({\widehat{\sigma}}_\Delta^j (\theta)\right)^2 \gets \sum_{k=1}^{K} p_k \Big(\big(\sigma_{\Delta_k}^j\big)^2 + \big(\Delta_k^j(\theta)  - \Delta^j\big)^2\Big)$
                \STATE $w_\alpha^j(\theta) \gets z_{1-\alpha/(2d)} \sqrt{\frac{\left({\widehat{\sigma}}_g^j\left(\theta\right)\right)^2}{N} + \frac{\left({\widehat{\sigma}}_\Delta^j (\theta)\right)^2}{n}}$
                \ENDFOR
		\end{algorithmic}
            \textbf{Output}: FL-prediction-powered confidence set $ \mathcal{C}_\alpha^{PP} = \left\{ \theta \in \Theta_{grid} : \left| {\widetilde{g}}^j(\theta) + {\widehat{\Delta}}^j(\theta) \right| \le w_\alpha^j(\theta), \forall j \in [d] \right\} $
	\end{algorithm}
\end{figure}
\begin{proposition}[Logistic regression]
\label{logistic-prop}
    Let $\theta^*$ be the logistic regression solution \eqref{logistic-goal}. Then, the prediction-powered confidence interval in Algorithm \ref{FL-logistic-alg} has valid coverage:
$$
    \liminf_{n,N\to\infty} P\left(\theta^* \in \mathcal{C}_\alpha^{PP}\right) \geq 1 - \alpha.
$$
\end{proposition}
\subsubsection{Linear regression}
Finally, we consider inference for linear regression:
\begin{equation}
\label{linear-goal}
    \theta^* = \arg\min_{\theta \in \mathbb{R}^d} \mathbb{E}[\ell_\theta(\widebar{X}_k^i, \widebar{Y}_k^i)] = \arg\min_{\theta \in \mathbb{R}^d} \mathbb{E}[\frac{1}{2}(\widebar{Y}_k^i - (\widebar{X}_k^i)^\top \theta)^2].
\end{equation}
The linear loss is differentiable and hence the optimality condition \eqref{convex-solution} is ensured. Its gradient is equal to $g_{\theta}(\widebar{X}_k^i, \widebar{Y}_k^i) = (\widebar{X}_k^i)^+(\widebar{X}_k^i\theta-\widebar{Y}_k^i)$, where $(\widebar{X}_k^i)^+$ is the pseudo-inverse matrix of $\widebar{X}_k^i$. Applying this in the definition of the prediction \eqref{practical-prediction} and rectifier \eqref{practical-rectifier}, we obtain $ \widetilde{g}(\theta) = \theta-\mathbb{E}_{k,i}\left[(\widetilde{X}_k^i)^+ f(\widetilde{X}_k^i)\right]$ and $\widehat{\Delta}=\mathbb{E}_{k,i}\left[(X_k^i)^+\left(f(X_k^i) - Y_k^i\right)\right]$. It is evident that $\widehat{\Delta}$ does not depend on the value of $\theta$. Consequently, we develop the linear regression algorithm employing the same strategy as that used for mean estimation. In Algorithm \ref{FL-linear-alg} we state a method for FL-prediction-powered linear regression and in Proposition \ref{linear-prop} we provide its guarantee.
\begin{figure}[!t]
	\begin{algorithm}[H]
		\caption{FL-prediction-powered linear regression estimation}
		\label{FL-linear-alg}
		\textbf{Input}: Labeled data $(X_k^i, Y_k^i)$, unlabeled features $\widetilde{X}_k^i$, datasize $\{N_k,N,n_k,n\}$, predictor $f$, coefficient $j^* \in [d]$
, error level $\alpha \in (0,1)$.
		
		\begin{algorithmic}[1] 
			\STATE Prediction-powered estimator: ${\widetilde{\theta}}^f \gets \sum_{k=1}^{K} p_k {\widetilde{\theta}}_k^f  $\\ $ = \sum_{k=1}^{K} p_k \frac{1}{N_k} \sum_{i=1}^{N_k}({\widetilde{X}}_k^i)^+ f({\widetilde{X}}_k^i)$.
			\STATE Empirical rectifier: $\widehat{\Delta}(\theta) \gets \sum_{k=1}^{K} p_k \widehat{\Delta}_k(\theta) $\\ $= \sum_{k=1}^{K} p_k \frac{1}{n_k}  \sum_{i=1}^{n_k} (X_k^i)^+ \left( f(X_k^i) - Y_k^i \right)$.
			\STATE Prediction-powered estimator: ${\widehat{\theta}}^{PP} \gets {\widetilde{\theta}}^f - \widehat{\Delta}$.
                \STATE “Sandwich” variance estimator for prediction:\\$\widetilde{\Sigma} \gets \sum_{k=1}^{K} p_k\frac{1}{N_k} \sum_{i=1}^{N_k} ( \widetilde{X}_k^i)^\top \widetilde{X}_k^i$\\
                $\widetilde{M}_k \gets \frac{1}{N_k} \sum_{i=1}^{N_k} \left( f(\widetilde{X}_k^i) - (\widetilde{X}_k^i)^\top \widetilde{\theta}_k^f \right)^2 \widetilde{X}_k^i (\widetilde{X}_k^i)^\top$\\
                $\widetilde{M} \gets \sum_{k=1}^{K} p_k \big( \widetilde{M}_k + ( \widetilde{\theta}_k^f - \widetilde{\theta}^f )^2 \big)$\\
                $\widetilde{V} \gets ( \widetilde{\Sigma} )^{-1} \widetilde{M} ( \widetilde{\Sigma} )^{-1}$
                \STATE “Sandwich” variance estimator for rectifier:\\$\Sigma \gets \sum_{k=1}^{K} p_k\frac{1}{n_k} \sum_{i=1}^{n_k} (X_k^i)^\top X_k^i$\\
                $M_k \gets \frac{1}{n_k} \sum_{i=1}^{n_k} \big( f(X_k^i) - Y_k^i - (X_k^i)^\top \widehat{\Delta}_k(\theta) \big)^2 X_k^i (X_k^i)^\top$\\
                $M \gets \sum_{k=1}^{K} p_k \big( M_k + ( \widehat{\Delta}_k(\theta) - \widehat{\Delta}(\theta) )^2 \big)$\\
                $V \gets ( \Sigma )^{-1} M ( \Sigma )^{-1}$
                \STATE $w_\alpha \gets z_{1-\alpha/2} \sqrt{\frac{\widetilde{V}_{j^\star j^\star}}{N} + \frac{V_{j^\star j^\star}}{n}}$
		\end{algorithmic}
            \textbf{Output}: FL-prediction-powered confidence set $ \mathcal{C}_\alpha^{PP} = \left( \widehat{\theta}_{j^\star}^{PP} \pm w_\alpha \right) $
	\end{algorithm}
\end{figure}
\begin{proposition}[Linear regression]
\label{linear-prop}
    Let $\theta^*$ be the linear regression solution \eqref{linear-goal} and fix $j^* \in [d]$. Then, the prediction-powered confidence interval in Algorithm \ref{FL-linear-alg} has valid coverage:
$$
    \lim \inf_{n,N \to \infty} P\left(\theta_{j^*}^* \in \mathcal{C}_\alpha^{PP}\right) \geq 1 - \alpha.
$$
\end{proposition}

\section{Performance Analysis}
\label{sec-experiment}
To evaluate the proposed algorithm, the experiments focused on the qualitative and quantitative analysis of general properties under the setup of our prototype system and the simulation of IID and Non-IID datasets. 
\subsection{Real tasks}
The dataset and statistical target $\theta^*$ for the real task are described in detail in \cite{angelopoulos2023prediction}. In the following, we will introduce the FL-PPI algorithm for the key parts.

\subsubsection{Galaxy classification}
The goal is to determine the demographics of galaxies with spiral arms, which are correlated with star formation in the discs of low-redshift galaxies and therefore contribute to the understanding of star formation in the Local Universe. Our focus is on estimating the fraction of galaxies with spiral arms. We then use the Algorithm \ref{FL-mean-alg} for the FL-prediction-powered mean estimation to construct intervals.

\subsubsection{Estimating deforestation in the Amazon}
The goal is to estimate the fraction of the Amazon rainforest lost between 2000 and 2015, using the Algorithm \ref{FL-mean-alg} to construct the FL-prediction-powered intervals.

\subsubsection{Relating protein structure and post-translational modifications}
The goal is to characterize whether various types of post-translational modifications (PTMs) occur more frequently in intrinsically disordered regions (IDRs) of proteins \cite{iakoucheva2004importance} by using structures predicted by AlphaFold \cite{jumper2021highly}.

We use the fact that the odds ratio, between whether or not a protein residue is part of an IDR, and whether or not it has a PTM, can be expressed as a function of two means:
$$
\theta^* = \frac{\mu_1 / (1 - \mu_1)}{\mu_0 / (1 - \mu_0)}
$$
Since Algorithm \ref{FL-mean-alg} can provide FL-prediction-powered confidence intervals 
$\mathcal{C}_0^{PP} = [l_0, u_0] $ and $\mathcal{C}_1^{PP} = [l_1, u_1]$ for the two means, $\mu_1$ and $\mu_0$, we can obtain the following confidence interval for the odds ratio function.
$$
\mathcal{C}_\alpha^{PP} = \left( \frac{l_1}{1 - l_1} \cdot \frac{1 - u_0}{u_0}, \frac{u_1}{1 - u_1} \cdot \frac{1 - l_0}{l_0} \right)
$$

\subsubsection{Distribution of gene expression levels}
The goal is to characterize how a population of promoter sequences affects gene expression, focusing on estimating the 0.5-quantiles of gene expression levels induced by native yeast promoters. We construct FL-prediction-powered confidence intervals on quantiles, specifically using the Algorithm \ref{FL-quantile-alg} where $q = 0.5$. 

\subsubsection{Relationship between income and private health insurance}
The goal is to investigate the quantitative effect of income on the procurement of private health insurance using US census data in 2019. We use a gradient-boosted tree \cite{chen2016scalable} trained on the previous year’s data to predict the health insurance indicator. We construct a FL-prediction-powered confidence interval on the logistic regression coefficient using the Algorithm \ref{FL-logistic-alg}.

\subsubsection{Relationship between age and income in a covariate-shifted population}
The goal is to investigate the relationship between age and income using US census data. We use the same dataset as in the previous experiment, but the features are age and sex, and the target is yearly income in dollars. We used a gradient-boosted tree \cite{chen2016scalable} trained on the previous year’s raw data to predict the income. We construct a prediction-powered confidence interval on the ordinary least squares regression coefficient using the Algorithm \ref{FL-linear-alg}.

\subsection{Setup}
\label{subsec-set}
To assess the overall performance of the proposed algorithm, we first conducted experiments using a networked prototype system with clients. We then recorded and analyzed the variations in the $\mathcal{C}_\alpha^{PP}$ metric of the proposed algorithm.

\subsubsection{Simulation of Dataset Distribution}
\label{sec-distribute}
We have a total dataset $\left(\widebar{X}_{k}^{i}, \widebar{Y}_{k}^{i},f({\widebar{X}}_k^i)\right)$, and set up two different Non-IID cases and a standard IID case to simulate the dataset distribution.
\begin{itemize}
	\item \textbf{Case 1} (IID): The samples from the total dataset are randomly and uniformly distributed to the individual clients.
	\item \textbf{Case 2} (Non-IID): The total dataset is sorted by the value of $f({\widebar{X}}_k^i)$ and then evenly distributed among the clients in that order.
	\item \textbf{Case 3} (Non-IID): The first half of the total dataset is randomly shuffled, while the second half is sorted based on the value of $f({\widebar{X}}_k^i)$. The samples are then evenly distributed among the clients.
\end{itemize}
\subsubsection{Control Parameters}
\label{control}
At the beginning of our experiments, we configured the prototype system with 5 nodes, where each node has an equally sized dataset, resulting in an total dataset partition of [1:1:1:1:1]. The proportion of labeled samples to the total number of samples in each dataset is $\lambda=0.1$.

In the subsequent experiments, we conducted ablation studies on three control parameters: the proportion of labeled samples, the total dataset partition, and the number of clients. The experimental results are presented in Sections \ref{sec-proportion}, \ref{sec-partition}, and \ref{sec-client}, respectively.

\subsection{Results}
\label{sec-result}
We first conduct experiments under the initial settings described in Section \ref{control}, and then analyze the control parameters separately: the proportion of labeled samples, the total dataset partition, and the number of clients.
\subsubsection{Prediction-powered confidence interval under \textbf{Case 1-3} with initial settings}
\label{sec-CPPs}
\begin{figure*}[!t]
	\centering
	\includegraphics[width=0.8\textwidth]{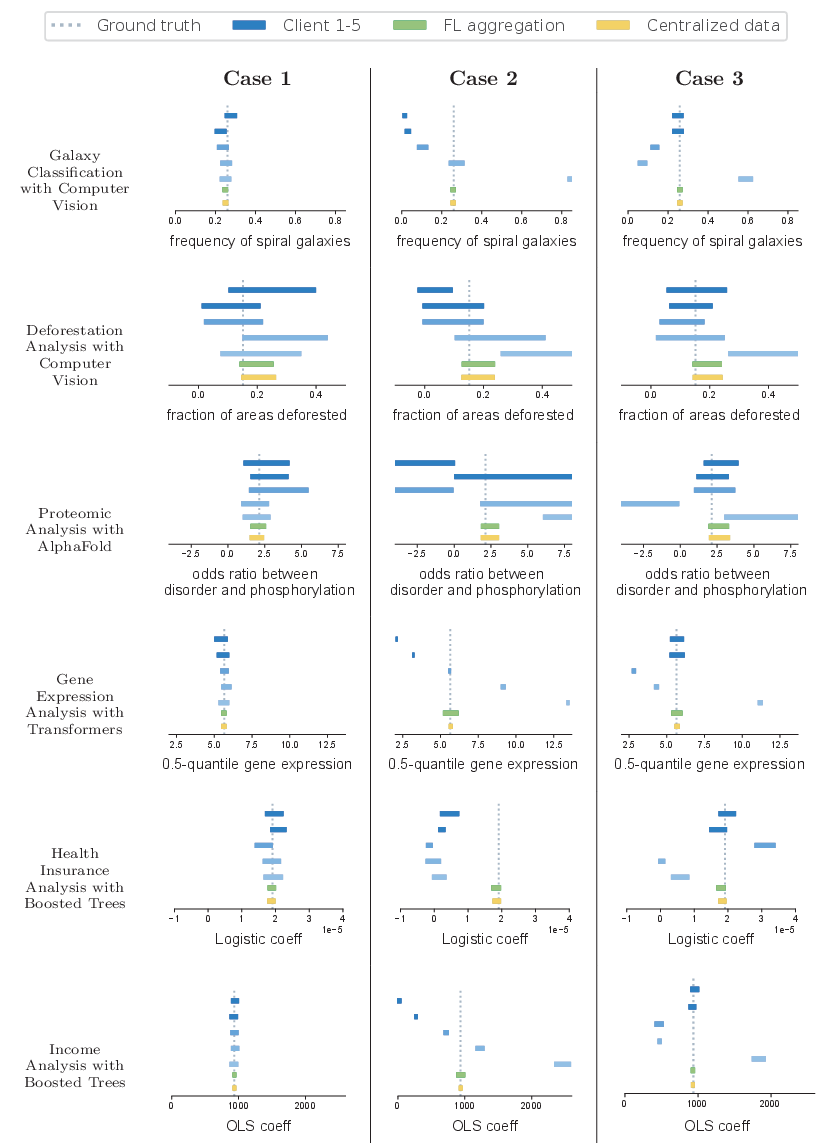}
	\caption{\textbf{Comparison of prediction-powered confidence interval at Client 1-5, FL aggregation and Centralized data.} Each row is a different application. Column 1 provides an introduction to the application, while columns 2-4 present \textbf{Case 1-3} as outlined in Section \ref{subsec-set}. In each figure, the prediction-powered confidence intervals at clients 1-5 are represented by blue gradient bars, with lighter shades indicating higher confidence levels.}
        \label{fig-interval-c1-c3}
\end{figure*}
In the initial experiment, we conducted real tasks in \textbf{Case 1-3}, and recorded the prediction-powered confidence intervals for both each client, federated aggregation and centralized data. It is important to clarify that federated aggregation does not transmit individual node dataset information (see Algorithms \ref{FL-mean-alg}-\ref{FL-linear-alg}), while centralized data directly computes the $\mathcal{C}_\alpha^{PP}$ for the entire dataset $\bigcup (\widebar{X}_k^i, \widebar{Y}_k^i)$. The results are shown in Figure \ref{fig-interval-c1-c3}, where the ground truth is directly calculated from total dataset as $\mathbb{E} [\widebar{Y}_k^i]$. Moreover, the prediction-powered confidence intervals ($\mathcal{C}_\alpha^{PP}$) for each client are depicted as gradient blue bars, the federated aggregation $\mathcal{C}_\alpha^{PP}$ is shown as a green bar, and the centralized data $\mathcal{C}_\alpha^{PP}$ is represented by a yellow bar.

From \textbf{Case 1} (IID dataset) in Figure \ref{fig-interval-c1-c3}, we can observe that the federated aggregation $\mathcal{C}_\alpha^{PP}$ for each real task successfully covers the ground truth. Moreover, it is narrower and closer to the centralized data $\mathcal{C}_\alpha^{PP}$ compared to the individual clients. These experimental results are consistent with our analysis of Eq. \eqref{eq-mean-example}.

In \textbf{Case 2} and \textbf{Case 3} (Non-IID dataset) shown in Figure \ref{fig-interval-c1-c3}, some clients' $\mathcal{C}_\alpha^{PP}$ intervals fail to cover the ground truth due to differences in sample feature distributions between local datasets and the total dataset. However, our FL-PPI algorithm still produces $\mathcal{C}_\alpha^{PP}$ intervals similar to those of the centralized data, demonstrating that FL-PPI can accurately represent the total dataset in mean estimation. For quantile estimation, in the fourth task, the significant differences in rectifiers (see Eq.\eqref{var-agg}) across clients cause the FL-PPI algorithm to produce a wider $\mathcal{C}_\alpha^{PP}$ interval (still covers the ground truth). For the logistic regression estimation, corresponding to the fourth real task, we observed that the $\mathcal{C}_\alpha^{PP}$ values on each client are skewed towards the ground truth in \textbf{Case 2}. This behavior is attributed to the fact that the logistic regression loss function is influenced not only by the distribution of $f({\widebar{X}}_k^i)$ but also by the parameter $\mu_\theta$.

\subsubsection{Impact of labeled sample proportion}
\label{sec-proportion}

\renewcommand{\arraystretch}{1.5}
\begin{table*}[!t]
	\centering
	\caption{In \textbf{Case 1}, the prediction-powered confidence interval $\mathcal{C}_\alpha^{PP}$ under different proportions of labeled data.}
	\label{tab-proportion}
	\centering
	\begin{tabular}{cccccccc}
		\hline
		Problem & Ground truth $\theta^*$ & Strategy & $\lambda=0.01$ & $\lambda=0.3$ & $\lambda=0.5$  & $\lambda=0.7$ & $\lambda=0.99$\\
		\hline
		~ \multirow{2}{*}{Galaxy classification}  &  \multirow{2}{*}{0.259} & Centralized & [0.220, 0.305] & [0.247, 0.263] & [0.249, 0.263] & [0.251, 0.263] & \bf{[0.254, 0.265]} \\
		~  &   & \bf{FL} & [0.217, 0.300] & [0.247, 0.263] & [0.249, 0.263] & \bf{[0.251, 0.263]} & [0.251, 0.268] \\
		\hline
		~ \multirow{2}{*}{Deforestation analysis}  &  \multirow{2}{*}{0.152} & Centralized & [0.091, 0.513] & [0.142, 0.205] & [0.142, 0.191] & [0.137, 0.178] & \bf{[0.135, 0.171]} \\
            ~  &   & \bf{FL} & [0.078, 0.475] & [0.145, 0.207] & [0.143, 0.191] & \bf{[0.137, 0.178]} & [0.132, 0.174] \\
		\hline
		~ \multirow{2}{*}{Proteomic analysis}  &  \multirow{2}{*}{2.131} & Centralized & [1.143, 5.660] & [1.803, 2.532] & [1.860, 0.2.488] & [1.846, 2.411] & \bf{[1.885, 2.419]} \\
            ~  &   & \bf{FL} & [1.242, 5.901] & [1.804, 2.535] & [1.856, 2.483] & \bf{[1.847, 2.414]} & [1.411, 3.244] \\
            \hline
		~ \multirow{2}{*}{Gene expression}  &  \multirow{2}{*}{5.650} & Centralized & [4.920, 5.817] & [5.513, 5.749] & [5.443, 5.668] & \bf{[5.469, 5.717]} & [5.522, 6.481] \\
            ~  &   & \bf{FL} & [4.909, 5.860] & [5.511, 5.751] & [5.441, 5.662] & \bf{[5.468, 5.716]} & [5.270, 6.531] \\
            \hline
		~ \multirow{2}{*}{Health insurance}  &  \multirow{2}{*}{1.913($10^{-5}$)} & Centralized & [1.599, 2.337] & [1.837, 1.980] & [1.847, 1.962] & [1.841, 1.941] & \bf{[1.870, 1.957]} \\
            ~  &   & \bf{FL} & [1.685, 2.480] & [1.838, 1.980] & [1.848, 1.963] & [1.841, 1.941] & \bf{[1.871, 1.958]} \\
            \hline
		~ \multirow{2}{*}{Income analysis}  &  \multirow{2}{*}{0.938($10^{3}$)} & Centralized & [0.853, 1.033] & [0.919, 0.951] & [0.923, 0.948] & [0.927, 0.949] & \bf{[0.930, 0.949]} \\
            ~  &   & \bf{FL} & [0.854, 1.033] & [0.919, 0.951] & [0.923, 0.948] & [0.927, 0.949] & \bf{[0.930, 0.949]} \\
		\hline
	\end{tabular}
\end{table*}
To further investigate the impact of the proportion $\lambda$ of labeled samples to the total sample size in each local dataset, we configured two extreme cases with $\lambda = [0.01, 0.99]$ and three standard cases with $\lambda = [0.3, 0.5, 0.7]$ under the scenario of \textbf{Case 1} (other control parameters fixed). The experimental results are presented in Table \ref{tab-proportion}.

From Table \ref{tab-proportion}, we can observe that as $\lambda$ increases from 0.01 to 0.7, the $\mathcal{C}_\alpha^{PP}$ narrows accordingly. To understand this phenomenon, we need to analyze Eq. \eqref{eq-mean-example}: as $n$ increases and $N$ decreases, the value of $\frac{\left(\widehat{\sigma}^{f-Y}\right)^2}{n}$ decreases while the value of $\frac{\left(\widehat{\sigma}^f\right)^2}{N}$ increases. Since the decrease in $\frac{\left(\widehat{\sigma}^{f-Y}\right)^2}{n}$ is greater than the increase in $\frac{\left(\widehat{\sigma}^f\right)^2}{N}$, the overall $w_\alpha$ value decreases, leading to a narrower $\mathcal{C}_\alpha^{PP}$. When $\lambda$ increases from 0.7 to 0.99, the changes in $\frac{\left(\widehat{\sigma}^{f-Y}\right)^2}{n}$ and $\frac{\left(\widehat{\sigma}^f\right)^2}{N}$ become more random, resulting in a $\mathcal{C}_\alpha^{PP}$ that can either narrow or widen unpredictably.

From Table \ref{tab-proportion} and Figure \ref{fig-interval-c1-c3}, we can see that when $\lambda = 0.1$, the $\mathcal{C}_\alpha^{PP}$ is already sufficiently narrow. This indicates that our FL-PPI algorithm requires only a small amount of labeled data to achieve statistically significant confidence intervals.

\subsubsection{Different total dataset partitions}
\label{sec-partition}

\begin{figure}[!t]
	\centering
	\includegraphics[width=3.5in]{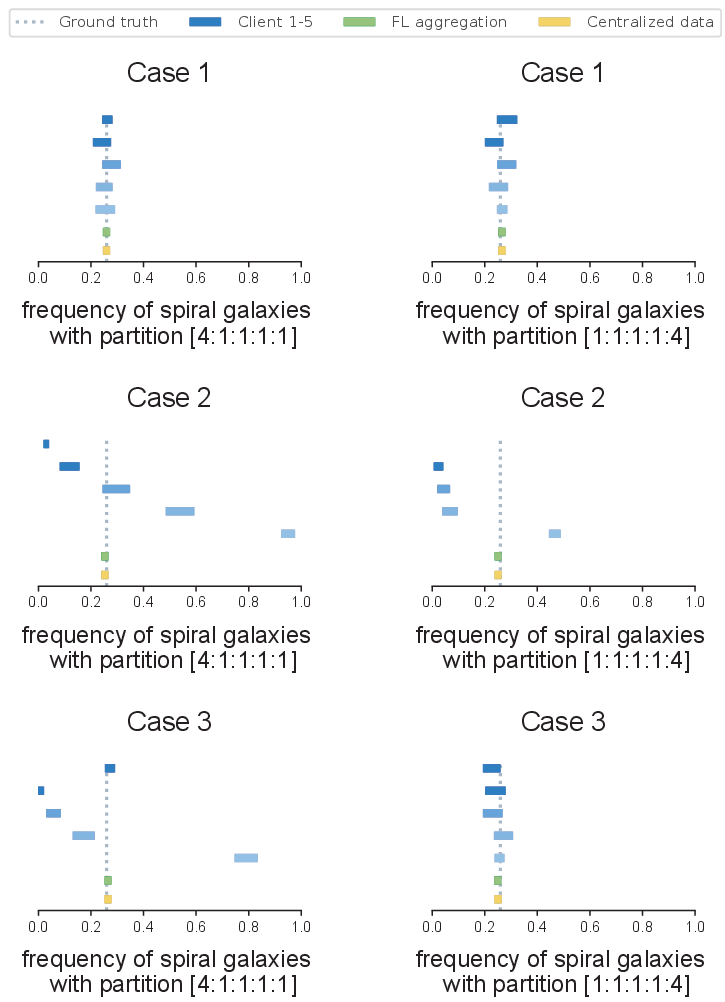}
	\caption{\textbf{Prediction-powered confidence intervals with different partition.} The rows represent scenarios from \textbf{Case 1} to \textbf{Case 3}, and the columns represent two different total dataset partition: [4:1:1:1:1] and [1:1:1:1:4].}
        \label{fig-partition}
\end{figure}
To observe the impact of varying sample sizes across 5 clients (Client 1-5) on the prediction-powered confidence intervals, we configured two different total dataset partitioning methods: [4:1:1:1:1] and [1:1:1:1:4] (other control parameters fixed). In partition [4:1:1:1:1], the first client holds the first half of the total dataset, while the remaining clients equally share the rest. In partition [1:1:1:1:4], the last client holds the second half of the total dataset, with the remaining clients equally sharing the rest. We conducted the `Galaxy Classification with Computer Vision' task in the scenarios of \textbf{Case 1-3}, the experimental results are presented in Figure \ref{fig-partition}.

From Figure \ref{fig-partition}, we can observe that in \textbf{Case 1-3}, an increase in the sample size on a client leads to a narrowing of its local $\mathcal{C}_\alpha^{PP}$, consistent with our analysis in Section \ref{example}. Furthermore, it is noted that in \textbf{Case 2}, with partition [1:1:1:1:4], the confidence interval for Client 1 is nearly [0, 0] and does not appear, as the small size of the data increases the likelihood that local data samples have $Y_1^i=0$ for all $i$ (thus $\widehat{\theta}^{PP}=0$, $\left(\widehat{\sigma}^f\right)^2=0$, and $\left(\widehat{\sigma}^{f-Y}\right)^2=0$). Lastly, neither of these partitions significantly affected the $\mathcal{C}_\alpha^{PP}$ of the FL-PPI algorithm, demonstrating the robustness of the algorithm.

\subsubsection{Varying number of clients}
\label{sec-client}

\begin{figure}[!t]
	\centering
	\includegraphics[width=3.5in]{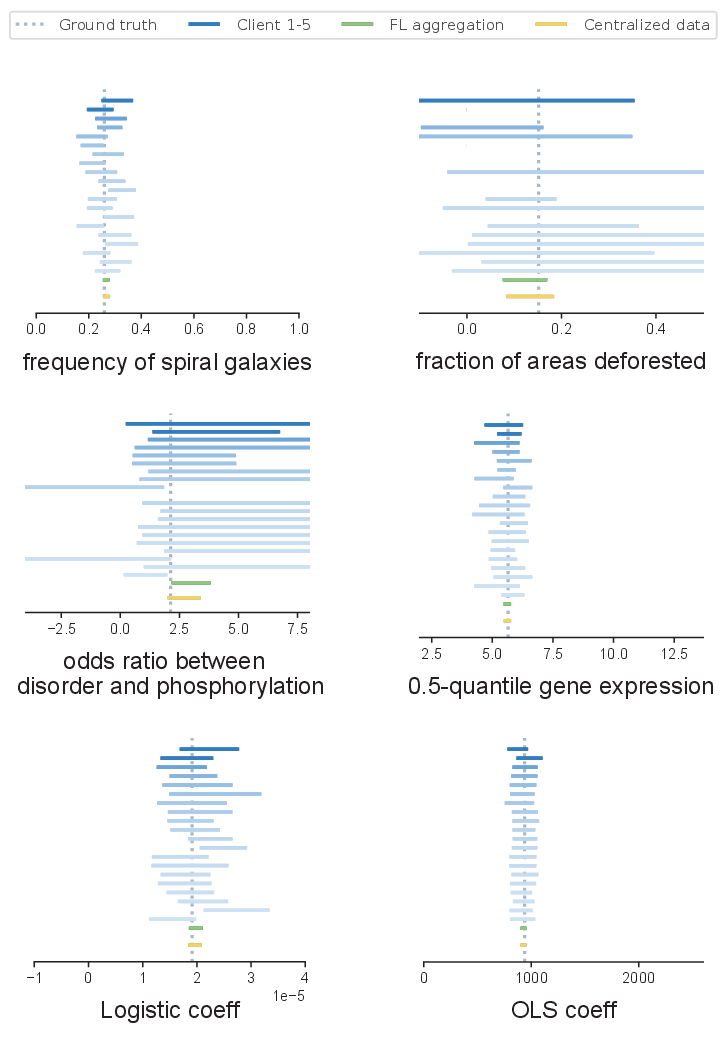}
	\caption{\textbf{Prediction-powered confidence intervals with 20 clients in Case 1.} Each subplot corresponds to a real task.}
        \label{fig-20clients}
\end{figure}
Keeping other parameters at their initial values, we expanded the number of clients to 20 to observe its impact on all real tasks under \textbf{Case 1}. The experimental results are shown in Figure \ref{fig-20clients}. The increase in the number of clients resulted in a reduced sample size per client, leading to a widening of the local CPP. In task `Deforestation Analysis with Computer Vision' and `Proteomic Analysis with AlphaFold', this even caused the $\mathcal{C}_\alpha^{PP}$ to shrink to [0, 0], making it disappear from the display, which is consistent with the observations discussed in Section \ref{sec-partition}. However, the increase in the number of clients had little impact on our FL-PPI algorithm. Its $\mathcal{C}_\alpha^{PP}$ remained the same width as that of the centralized data (cover the true value).

\section{Conclusion and The Future Work}
\label{sec-conclusion}
To address the challenge of `data silos' in Prediction-Powered Inference (PPI), this paper proposes the Federated Prediction-Powered Inference (Fed-PPI) framework. Fed-PPI enables decentralized experimental data to contribute to statistically valid conclusions without sharing private information. We introduced algorithms for common statistical problems within this framework and provided a theoretical analysis of their performance. Extensive experiments demonstrate the statistical validity of the confidence intervals obtained through Fed-PPI, highlighting its potential to overcome data sharing limitations in real-world scenarios. Future work will focus on optimizing computational efficiency and expanding the theoretical framework to various statistical applications.



\normalem
\bibliographystyle{IEEEtran}
\bibliography{IEEEabrv,myrefs}

\clearpage
\onecolumn

{\appendices
\section{PPI Parameter Estimation: Federated Aggregation vs. Direct Computation}
\label{appendix1}
\begin{proposition}
	Based on our definitions of the two PPI parameter computation methods for $g(\theta)$ in Eq. \eqref{original-prediction} and Eq. \eqref{additional-agg}, we have 
    $$
    \mathbb{E}_{k,i} = \mathbb{E}_{k} \left[\mathbb{E}_{i} \left[g_\theta \left( {\widebar{X}}_k^i, f({\widebar{X}}_k^i) \right) \right]\right] = \mathbb{E} \left[ \bigcup g_\theta (\widebar{X}_k^i, \widebar{Y}_k^i) \right] = \mathbb{E}_{\bigcup},
    $$ 
    thus $\mathbb{E}_{k,i}$ is equivalent to the direct computation of the PPI parameters on the entire dataset $\bigcup (\widebar{X}_k^i, \widebar{Y}_k^i)$.
\end{proposition}
\begin{proof}
	In order to proceed to the proof, we first rewrite Eq. \eqref{original-prediction} and Eq. \eqref{additional-agg}
	$$
		\begin{aligned}
			\mathbb{E}_{k,i} &= \sum_{k=1}^{K} p_k \frac{1}{m_k} \sum_{i=1}^{m_k} g_\theta \left( {\widebar{X}}_k^i, f({\widebar{X}}_k^i) \right) = \sum_{k=1}^{K} \frac{m_k}{\sum_{k=1}^{K} m_k} \frac{1}{m_k} \sum_{i=1}^{m_k} g_\theta \left( {\widebar{X}}_k^i, f({\widebar{X}}_k^i) \right) = \frac{1}{\sum_{k=1}^{K} m_k} \sum_{k=1}^{K} \sum_{i=1}^{m_k} g_\theta \left( {\widebar{X}}_k^i, f({\widebar{X}}_k^i) \right) \\
            &= \mathbb{E} \left[ \bigcup g_\theta (\widebar{X}_k^i, \widebar{Y}_k^i) \right] = \mathbb{E}_{\bigcup},
		\end{aligned}
	$$
	where the third term in the equation is due to $p_k := \frac{m_k}{\sum_{k=1}^{K} m_k}$, while the fourth term arises because $\frac{1}{\sum_{k=1}^{K} m_k}$ is a constant and is not influenced by $k$.
 
    That completes the proof.
\end{proof}
\section{Proof of Theorems}
\label{appendix2}
\setcounter{theorem}{0}
\setcounter{proposition}{0}
\setcounter{lemma}{0}

\subsection{Convex estimation}
\begin{theorem}
    Suppose that the convex estimation problem is nondegenerate as in \eqref{convex-solution}. Fix $\alpha \in (0,1)$ and $\Delta(\theta) \in (0,\alpha)$. Suppose that, for any $\theta \in \mathbb{R}^d$, we can construct $\mathcal{T}_{\alpha-\delta}$ and $\mathcal{R}_\delta(\theta)$ satisfying
    \begin{equation}
        \left\{
            \begin{aligned}
                &P\left(g(\theta)\in\mathcal{T}_{\alpha-\delta}(\theta)\right)\geq1-(\alpha-\delta)\\
                &P(\Delta(\theta)\in\mathcal{R}_\delta(\theta))\geq1-\delta
            \end{aligned} 
        \right.
    \end{equation}
    Let $\mathcal{C}_\alpha^{PP}=\{\theta:0\in\mathcal{R}_\delta(\theta)+\mathcal{T}_{\alpha-\delta}(\theta)\}$, where $+$ denotes the Minkowski sum. Then,
    \begin{equation}
        P(\theta^\ast\in\mathcal{C}_\alpha^{PP})\geq1-\alpha
    \end{equation}
\end{theorem}
\begin{proof}
Consider the event $E=\{\Delta(\theta^{*})\in\mathcal{R}_{\delta}(\theta^{*})\}\cap\{g(\theta^{*})\in\mathcal{T}_{\alpha-\delta}(\theta^{*})\}$. By a union bound, $P(E)\geq 1-\alpha$. On the event $E$, we have that
$$
\begin{aligned}
    \mathbb{E}_{k,i}\left[g_{\theta^\ast}({\widebar{X}}_k^i,{\widebar{Y}}_k^i)\right] &= \mathbb{E}_{k,i} \left[  g_{\theta^\ast}({\widebar{X}}_k^i, {\widebar{Y}}_k^i) - g_{\theta^\ast}({\widebar{X}}_k^i, f({\widebar{X}}_k^i)) + g_{\theta^\ast}({\widebar{X}}_k^i, f({\widebar{X}}_k^i)) \right] \\
    &= \mathbb{E}_{k,i} \left[ g_{\theta^\ast}({\widebar{X}}_k^i, {\widebar{Y}}_k^i) - g_{\theta^\ast}({\widebar{X}}_k^i, f({\widebar{X}}_k^i)) \right] + \mathbb{E}_{k,i} \left[ g_{\theta^\ast}({\widebar{X}}_k^i, f({\widebar{X}}_k^i)) \right] \\
    &=\Delta(\theta^{*}) + g(\theta^{*}) \in \mathcal{R}_{\delta}(\theta^{*}) + \mathcal{T}_{\alpha-\delta}(\theta^{*})
\end{aligned}  
$$
Invoking the nondegeneracy condition which ensures $\mathbb{E}_{k,i}\left[g_{\theta^\ast}({\widebar{X}}_k^i,{\widebar{Y}}_k^i)\right] = 0$, thus we have
$$
    P\big(0 \in \mathcal{R}_{\delta}(\theta^{*}) + \mathcal{T}_{\alpha-\delta}(\theta^{*})\big) \geq 1-\alpha
$$
where it shows that $\theta^{*} \in \mathcal{C}_\alpha^{PP}$ with probability at least $1-\alpha$, thus
$$
    P(\theta^\ast\in\mathcal{C}_\alpha^{PP})\geq1-\alpha
$$
That completes the proof.
\end{proof}

\subsection{Convex estimation: asymptotic version}
\begin{theorem}
Suppose that the convex estimation problem is nondegenerate as in \eqref{convex-solution}. Denoting by $g^j(x, y)$ the $j$-th coordinate of $g(x, y)$. Fix $\alpha \in (0,1)$ and $j \in [d]$. For all $\theta \in \mathbb{R}^d$, define
    \begin{equation}
        \left\{
            \begin{aligned}
                &\widetilde{g}^{j}(\theta) =: \sum_{k=1}^{K} p_k \frac{1}{N_k} \sum_{i=1}^{N_k} g_\theta\big(\widetilde{X}_k^{i,j}, f(\widetilde{X}_k^{i,j})\big)\\
                &\widehat{\Delta}^{j}(\theta) =: \sum_{k=1}^{K} p_k \frac{1}{n_k} \sum_{i=1}^{n_k} \left( g_\theta(X_k^{i,j}, Y_k^{i,j}) - g_\theta(X_k^{i,j}, f(X_k^{i,j})) \right)
            \end{aligned} 
        \right.
    \end{equation}
Further, define $\left({\widehat{\sigma}}_{g}^j\left(\theta\right)\right)^2$ be the variance of $g_\theta\big(\widetilde{X}_k^i, f(\widetilde{X}_k^i)\big)$ values, and $\left({\widehat{\sigma}}_\Delta^j(\theta)\right)^2$ be the variance of $ g_\theta(X_k^i, Y_k^i) - g_\theta(X_k^i, f(X_k^i))$ values. Let $w_\alpha^j(\theta) = z_{1-\alpha/(2p)} \sqrt{\frac{\left({\widehat{\sigma}}_g^j\left(\theta\right)\right)^2}{N} + \frac{\left({\widehat{\sigma}}_\Delta^j(\theta)\right)^2}{n}}$ and $ \mathcal{C}_\alpha^{PP} = \left\{ \theta : \left| {\widetilde{g}}^j(\theta) + {\widehat{\Delta}}^j(\theta) \right| \le w_\alpha^j(\theta), \forall j \in [d] \right\} $.

Then, we have
$$\liminf_{n,N\to\infty}P(\theta^*\in\mathcal{C}_\alpha^{\mathrm{PP}})\geq1-\alpha.$$
\end{theorem}

\begin{proof}
    For each $j\in[d]$ of the dataset $\left(\widebar{X}_{k}^{i,j}, \widebar{Y}_{k}^{i,j},f({\widebar{X}}_k^{i,j})\right) \in (\mathcal{X} \times \mathcal{Y})^{m_k}$, we have 
    $$
        \Delta^j(\theta^{*}) = \mathbb{E}_{k,i} \left[ g_{\theta^\ast}({\widebar{X}}_k^{i,j}, {\widebar{Y}}_k^{i,j}) - g_{\theta^\ast}({\widebar{X}}_k^{i,j}, f({\widebar{X}}_k^{i,j})) \right]; \quad g^j(\theta^{*}) = \mathbb{E}_{k,i} \left[ g_{\theta^\ast}({\widebar{X}}_k^{i,j}, f({\widebar{X}}_k^{i,j})) \right]
    $$
    for all data sample $i$ at client $k$. Then, the central limit theorem implies that
$$\sqrt{n}(\widehat{\Delta}^j(\theta^{*}) - \Delta^j(\theta^{*}))\Rightarrow\mathcal{N}(0,(\sigma_{\Delta}^{j}(\theta^{*}))^{2}); \quad\sqrt{N}(\widetilde{g}^{j}(\theta^{*}) - g^j(\theta^{*}))\Rightarrow\mathcal{N}(0,(\sigma_{g}^{j}(\theta^{*}))^{2})$$

Therefore, by Slutsky’s theorem, we get
$$\begin{aligned}
&\quad\sqrt{N}(\widehat{\Delta}^j(\theta^{*})+\widetilde{g}^{j}(\theta^{*})-({\Delta}^j(\theta^{*})+{g}^{j}(\theta^{*}))) =\sqrt{n}(\widehat{\Delta}^j(\theta^{*}) - \Delta^j(\theta^{*}))\sqrt{\frac{N}{n}}+\sqrt{N}(\widetilde{g}^{j}(\theta^{*}) - g^j(\theta^{*})) \\
&\Rightarrow\mathcal{N}\left(0,(\sigma_{\Delta}^{j}(\theta^{*}))^{2}\frac{N}{n}+(\sigma_{g}^{j}(\theta^{*}))^{2}\right) = \mathcal{N}\left(0, (\widehat{\sigma}^{j})^{2}\right).
\end{aligned}$$

where we defined $(\widehat{\sigma}^{j})^{2}=(\sigma_{\Delta}^{j}(\theta^{*}))^{2}\frac{N}{n}+(\sigma_{g}^{j}(\theta^{*}))^{2}$. This in turn implies
\begin{equation}                
    \label{eq-theo2-1}
    \liminf_{n,N\to\infty}P\left(\left|\widehat{\Delta}^j(\theta^{*})+\widetilde{g}^{j}(\theta^{*})-({\Delta}^j(\theta^{*})+{g}^{j}(\theta^{*}))\right|\leq z_{1-\alpha/(2p)}\frac{\widehat{\sigma}^{j}}{\sqrt{N}}\right)\geq 1-\alpha
\end{equation}

Now notice that
\begin{equation}
    \label{eq-theo2-2}
        {\Delta}^j(\theta^{*})+{g}^{j}(\theta^{*}) = \mathbb{E}_{k,i}\left[g_{\theta^\ast}({\widebar{X}}_k^{i,j}, {\widebar{Y}}_k^{i,j}) - g_{\theta^\ast}({\widebar{X}}_k^{i,j}, f({\widebar{X}}_k^{i,j})) + g_{\theta^\ast}({\widebar{X}}_k^{i,j}, f({\widebar{X}}_k^{i,j})) \right] = \mathbb{E}[g_{\theta^\ast}({\widebar{X}}_k^{i,j}, {\widebar{Y}}_k^{i,j})] = 0,
\end{equation}
where the last step follows by the nondegeneracy condition, and
\begin{equation}
    \label{eq-theo2-3}
     \frac{\widehat{\sigma}^{j}}{\sqrt{N}} = \frac{\sqrt{(\sigma_{\Delta}^{j}(\theta^{*}))^{2}\frac{N}{n}+(\sigma_{g}^{j}(\theta^{*}))^{2}}}{\sqrt{N}} = \sqrt{\frac{(\sigma_{\Delta}^{j}(\theta^{*}))^{2}}{n}+\frac{(\sigma_{g}^{j}(\theta^{*}))^{2}}{N}}
\end{equation}
Substitute Eq. \eqref{eq-theo2-2} and \eqref{eq-theo2-3} back into equation Eq. \eqref{eq-theo2-1}, we get 
$$
    \liminf_{n,N\to\infty} P\left(\left| \widehat{\Delta}^j(\theta^{*})+\widetilde{g}^{j}(\theta^{*}) \right|\leq z_{1-\alpha/(2p)}\sqrt{\frac{(\sigma_{\Delta}^{j}(\theta^{*}))^{2}}{n}+\frac{(\sigma_{g}^{j}(\theta^{*}))^{2}}{N}}, \forall j\in[d]\right)\geq 1-\alpha.
$$
and
$$
    \liminf_{n,N\to\infty}P(\theta^*\in\mathcal{C}_\alpha^{\mathrm{PP}})\geq1-\alpha.
$$
That completes the proof.
\end{proof}

\subsection{General risk minimization: finite population}
\begin{theorem}
Fix $\alpha \in (0,1)$ and $\Delta(\theta) \in (0,\alpha)$. Suppose that, for any $\theta \in \Theta$, we can construct $\left(\mathcal{R}_{\delta/2}^l(\theta),\mathcal{R}_{\delta/2}^u(\theta)\right)$ and $ \left(\mathcal{T}_{\frac{\alpha-\delta}{2}}^l(\theta),\mathcal{T}_{\frac{\alpha-\delta}{2}}^u(\theta)\right)$ such that
    \begin{equation}
        \left\{
            \begin{aligned}
                &P\big(\Delta(\theta) \le \mathcal{R}_{\delta/2}^u(\theta)\big) \geq 1 - \delta/2\\
                &P\big(\Delta(\theta) \geq \mathcal{R}_{\delta/2}^l(\theta)\big) \geq 1 - \delta/2
            \end{aligned} 
        \right.
    \end{equation}
    and 
    $$
        \left\{
            \begin{aligned}
                &P\big(\widetilde{L}^f(\theta) - \mathbb{E}_{k,i} \left[\ell_\theta(\widetilde{X}_k^i, f({\widetilde{X}}_k^i))\right] \le \mathcal{T}_{\frac{\alpha - \delta}{2}}^u(\theta)\big) \geq 1 - \frac{\alpha - \delta}{2}\\
                &P\big(\widetilde{L}^f(\theta) - \mathbb{E}_{k,i} \left[\ell_\theta(\widetilde{X}_k^i, f({\widetilde{X}}_k^i))\right] \geq \mathcal{T}_{\frac{\alpha - \delta}{2}}^l(\theta)\big) \geq 1 - \frac{\alpha - \delta}{2}
            \end{aligned} 
        \right.
    $$
    Let 
    $$\mathcal{R}_{\delta/2}^d(\theta)=\mathcal{R}_{\delta/2}^u(\widetilde{\theta}^f)- \mathcal{R}_{\delta/2}^l(\theta),\quad \mathcal{T}_{\frac{\alpha - \delta}{2}}^d(\theta)=\mathcal{T}_{\frac{\alpha - \delta}{2}}^u(\theta) - \mathcal{T}_{\frac{\alpha - \delta}{2}}^l(\widetilde{\theta}^f)$$
    $$
        \mathcal{C}_\alpha^{PP} = \left\{ \theta \in \Theta : \widetilde{L}^f(\theta) \leq L^f(\widetilde{\theta}^f) + \mathcal{R}_{\delta/2}^d(\theta) + \mathcal{T}_{\frac{\alpha - \delta}{2}}^d(\theta) \right\}
    $$
    Then, we have
    $$P(\theta^* \in \mathcal{C}_\alpha^{PP}) \geq 1 - \alpha$$
\end{theorem}
\begin{proof}
    Define
    $$
        L(\theta) = \mathbb{E}_{k,i}[\ell_{\theta}(X_{k}^i, Y_{k}^i)], \quad L^f(\theta) = \mathbb{E}_{k,i}[\ell_{\theta}(X_{k}^i, f(X_{k}^i))].
    $$
    By the definition of $\theta^* = \arg \min_{\theta \in \mathbb{R}^d} \mathbb{E} \left[ \ell_\theta \left( \widebar{X}_k^i, \widebar{Y}_k^i \right) \right]$, we have
    $$
        \begin{aligned}
            \widetilde{L}^f(\theta^*) &= (\widetilde{L}^f(\theta^*) - L(\theta^*)) + (L(\theta^*) - L(\widetilde{\theta}^f)) + (L(\widetilde{\theta}^f) - \widetilde{L}^f(\widetilde{\theta}^f)) + \widetilde{L}^f(\widetilde{\theta}^f)\\
            &\leq (\widetilde{L}^f(\theta^*) - L(\theta^*)) + (L(\widetilde{\theta}^f) - \widetilde{L}^f(\widetilde{\theta}^f)) + \widetilde{L}^f(\widetilde{\theta}^f).
        \end{aligned}
    $$
    By applying the validity of the confidence bounds, a union bound implies that with probability $1 - \alpha$ we have
    $$
        \begin{aligned}
            \widetilde{L}^f(\theta^*) &\leq (L^f(\theta^*) - L(\theta^*)) + (L(\widetilde{\theta}^f) - L^f(\widetilde{\theta}^f)) + \widetilde{L}^f(\widetilde{\theta}^f) + T^u_{\frac{\alpha-\delta}{2}}(\theta^*) - T^l_{\frac{\alpha-\delta}{2}}(\widetilde{\theta}^f) \\
            &= -\Delta_{\theta^*} + \Delta_{\widetilde{\theta}^f} + \widetilde{L}^f(\widetilde{\theta}^f) + T^u_{\frac{\alpha-\delta}{2}}(\theta^*) - T^l_{\frac{\alpha-\delta}{2}}(\widetilde{\theta}^f) \\
            &\leq -R_{\frac{\delta}{2}}(\theta^*) + R_{\frac{\delta}{2}}(\widetilde{\theta}^f) + \widetilde{L}^f(\widetilde{\theta}^f) + T^u_{\frac{\alpha-\delta}{2}}(\theta^*) - T^l_{\frac{\alpha-\delta}{2}}(\widetilde{\theta}^f).
        \end{aligned}
    $$
    Therefore, with probability $1 - \alpha$ we have that $\theta^*\in\mathcal{C}_\alpha^{\mathrm{PP}}$, as desired. That completes the proof.    
\end{proof}

\section{Proof of Algorithms' Proposition}
\label{appendix3}
\subsection{Mean estimation}	
\begin{proposition}
    Let $\theta^*$ be the mean outcome \eqref{mean-goal}. Then, the prediction-powered confidence interval in Algorithm \ref{FL-mean-alg} has valid coverage:
$$
    \liminf_{n,N\to\infty} P\left(\theta^* \in \mathcal{C}_\alpha^{PP}\right) \geq 1 - \alpha.
$$
\end{proposition}
\begin{proof}
We show that the prediction-powered confidence set constructed in Algorithm \ref{FL-mean-alg} is a special case of the FL-prediction-powered confidence set constructed in Theorem \ref{asymptotic}. The proof then follows directly by the guarantee of Theorem \ref{asymptotic}.

Since $g_\theta \left( \widebar{Y}_k^i \right) = \theta - \widebar{Y}_k^i$, we have
$$
    \widetilde{g}(\theta) = \theta-\mathbb{E}_{k,i}\left[f(\widetilde{X}_k^i)\right]; \quad \widehat{\Delta}(\theta) = \mathbb{E}_{k,i}\left[f(X_k^i) - Y_k^i\right]
$$
Therefore, the set $\mathcal{C}_\alpha^{PP}$ from Theorem \ref{asymptotic} can be written as
$$
    \begin{aligned}
        \mathcal{C}_\alpha^{PP} &= \left\{ \theta : \left| \widetilde{g}(\theta) + \widehat{\Delta}(\theta) \right| \leq w_{\alpha}(\theta) \right\}\\
        &= \left\{ \theta : \left| \theta - \sum_{k=1}^{K} {p_k \frac{1}{N_k} \sum_{i=1}^{N_k} f({\widetilde{X}}_k^i)} + \sum_{k=1}^{K} {p_k \frac{1}{n_k} \sum_{i=1}^{n_k} (f(X_k^i) - Y_k^i)} \right| \leq w_{\alpha}(\theta) \right\} \\
        &= \sum_{k=1}^{K} p_k \Big( \frac{1}{N_k} \sum_{i=1}^{N_k} f({\widetilde{X}}_k^i) - \frac{1}{n_k} \sum_{i=1}^{n_k} (f(X_k^i) - Y_k^i) \Big)  \pm w_\alpha(\theta).
    \end{aligned}
$$
This is exactly the set constructed in Algorithm \ref{FL-mean-alg}.
\end{proof}

\subsection{Quantile estimation}
\begin{proposition}
    Let $\theta^*$ be the q-quantile \eqref{quantile-goal}. Then, the prediction-powered confidence interval in Algorithm \ref{FL-quantile-alg} has valid coverage:
$$
    \liminf_{n,N\to\infty} P\left(\theta^* \in \mathcal{C}_\alpha^{PP}\right) \geq 1 - \alpha.
$$
\end{proposition}
\begin{proof}
    Since $ g_\theta(\widebar{Y}_k^i) = -q + \mathbbm{1}\{\widebar{Y}_k^i \leq \theta\} $, we have 
    $$
        \widetilde{g}(\theta) =\widetilde{F}(\theta)-q; \quad \widehat{\Delta}(\theta) = \mathbb{E}_{k,i}\left[\mathbbm{1}\left\{Y_k^i \le \theta \right\} - \mathbbm{1}\left\{f(X_k^i) \le \theta \right\} \right]
    $$
    where $\widetilde{F}(\theta) = \mathbb{E}_{k,i}\left[\mathbbm{1}\left\{f({\widetilde{X}}_k^i) \le \theta \right\}\right] = \sum_{k=1}^{K} p_k \frac{1}{N_k} \sum_{i=1}^{N_k} \mathbbm{1}\left\{ f({\widetilde{X}}_k^i) \le \theta \right\}$. Therefore, the set $\mathcal{C}_\alpha^{PP}$ from Theorem \ref{asymptotic} can be written as
    $$
    \begin{aligned}
        \mathcal{C}_\alpha^{PP} &= \left\{ \theta : \left| \widetilde{g}(\theta) + \widehat{\Delta}(\theta) \right| \leq w_{\alpha}(\theta) \right\}\\
        &= \left\{ \theta : \left| \sum_{k=1}^{K} p_k \frac{1}{N_k} \sum_{i=1}^{N_k} \mathbbm{1}\left\{ f({\widetilde{X}}_k^i) \le \theta \right\} + \sum_{k=1}^{K} p_k\frac{1}{n_k} \sum_{i=1}^{n_k} \left( \mathbbm{1}\left\{ Y_k^i \le \theta \right\} - \mathbbm{1}\left\{ f(X_k^i) \le \theta \right\} \right) - q \right| \leq w_{\alpha}(\theta) \right\}.
    \end{aligned}
    $$
    This is exactly the set constructed in Algorithm \ref{FL-quantile-alg}. Therefore, the guarantee of Proposition \ref{quantile-prop} follows by the guarantee of Theorem \ref{asymptotic}.
\end{proof}

\subsection{Logistic regression}
\begin{proposition}
    Let $\theta^*$ be the logistic regression solution \eqref{logistic-goal}. Then, the prediction-powered confidence interval in Algorithm \ref{FL-logistic-alg} has valid coverage:
$$
    \liminf_{n,N\to\infty} P\left(\theta^* \in \mathcal{C}_\alpha^{PP}\right) \geq 1 - \alpha.
$$
\end{proposition}
\begin{proof}
    Since $ g_{\theta}(x, y) = -yx + x\mu_{\theta}(x) $, we have
    $$
        \widetilde{g}(\theta)=\mathbb{E}_{k,i}\left[\widetilde{X}_k^{(i,j)}\big(\mu_\theta(\widetilde{X}_k^i)-f({\widetilde{X}}_k^i)\big)\right]; \quad \widehat{\Delta}=\mathbb{E}_{k,i}\left[X_k^{(i,j)}\left(f(X_k^i) - Y_k^i\right)\right]
    $$
    Therefore, the set $\mathcal{C}_\alpha^{PP}$ from Theorem \ref{asymptotic} can be written as 
    $$
    \begin{aligned}
        \mathcal{C}_\alpha^{PP} &= \left\{ \theta : \left| \widetilde{g}(\theta) + \widehat{\Delta}(\theta) \right| \leq w_{\alpha}(\theta) \right\}\\
        &= \left\{ \theta : \left| \sum_{k=1}^{K} p_k \frac{1}{N_k} \sum_{i=1}^{N_k} \widetilde{X}_k^{(i,j)}\big(\mu_\theta(\widetilde{X}_k^i)-f({\widetilde{X}}_k^i)\big) + \sum_{k=1}^{K} p_k\frac{1}{n_k} \sum_{i=1}^{n_k} X_k^{(i,j)}\left(f(X_k^i) - Y_k^i\right) \right| \leq w_{\alpha}(\theta) \right\}.
    \end{aligned}
    $$
    This is exactly the set constructed in Algorithm \ref{FL-logistic-alg}. Therefore, the guarantee of Proposition \ref{logistic-prop} follows by the guarantee of Theorem \ref{asymptotic}.
\end{proof}

\subsection{Linear regression}
\begin{proposition}
    Let $\theta^*$ be the linear regression solution \eqref{linear-goal} and fix $j^* \in [d]$. Then, the prediction-powered confidence interval in Algorithm \ref{FL-linear-alg} has valid coverage:
$$
    \lim \inf_{n,N \to \infty} P\left(\theta_{j^*}^* \in \mathcal{C}_\alpha^{PP}\right) \geq 1 - \alpha.
$$
\end{proposition}
\begin{proof}
    The proof follows a similar pattern as the Proposition \ref{mean-prop}. Since $g_{\theta}(\widebar{X}_k^i, \widebar{Y}_k^i) = (\widebar{X}_k^i)^+(\widebar{X}_k^i\theta-\widebar{Y}_k^i)$, we have
    $$
         \widetilde{g}(\theta) = \theta-\mathbb{E}_{k,i}\left[(\widetilde{X}_k^i)^+ f(\widetilde{X}_k^i)\right]; \quad \widehat{\Delta}=\mathbb{E}_{k,i}\left[(X_k^i)^+\left(f(X_k^i) - Y_k^i\right)\right].
    $$
    Therefore, the set $\mathcal{C}_\alpha^{PP}$ from Theorem \ref{asymptotic} can be written as
$$
    \begin{aligned}
        \mathcal{C}_\alpha^{PP} &= \left\{ \theta : \left| \widetilde{g}(\theta) + \widehat{\Delta}(\theta) \right| \leq w_{\alpha}(\theta) \right\}\\
        &= \left\{ \theta : \left| \theta - \sum_{k=1}^{K} {p_k \frac{1}{N_k} \sum_{i=1}^{N_k} (\widetilde{X}_k^i)^+ f(\widetilde{X}_k^i)} + \sum_{k=1}^{K} {p_k \frac{1}{n_k} \sum_{i=1}^{n_k} (X_k^i)^+\left(f(X_k^i) - Y_k^i\right)} \right| \leq w_{\alpha}(\theta) \right\} \\
        &= \sum_{k=1}^{K} p_k \Big( \frac{1}{N_k} \sum_{i=1}^{N_k} (\widetilde{X}_k^i)^+ f(\widetilde{X}_k^i) - \frac{1}{n_k} \sum_{i=1}^{n_k} (X_k^i)^+\left(f(X_k^i) - Y_k^i\right) \Big)  \pm w_\alpha(\theta).
    \end{aligned}
$$
This is exactly the set constructed in Algorithm \ref{FL-linear-alg}, which completes the proof.
\end{proof}
}

\vfill

\end{document}